\documentclass[12pt,a4paper]{article}

\usepackage{arxiv}

\usepackage[table]{xcolor}
\usepackage{mathptmx}
\usepackage{microtype}

\usepackage{algpseudocode}
\usepackage{algorithm}

\usepackage{amsmath}
\usepackage{amssymb}
\usepackage{amsthm}
\usepackage{colortbl}
\newtheorem{theorem}{Theorem}[section]
\newtheorem{proposition}[theorem]{Proposition}
\newtheorem{lemma}[theorem]{Lemma}
\newtheorem{corollary}[theorem]{Corollary}
\newtheorem{definition}[theorem]{Definition}
\newtheorem{assumption}[theorem]{Assumption}
\newtheorem{remark}[theorem]{Remark}

\usepackage{booktabs,array,multirow,colortbl}
\usepackage{graphicx}
\graphicspath{{./}}

\usepackage{tikz}
\usetikzlibrary{shapes.geometric,arrows.meta,positioning,calc,backgrounds,fit,matrix,patterns}
\usepackage{pgfplots}
\pgfplotsset{compat=1.18}
\usepackage{pgfplotstable}
\usepackage{enumitem,caption,subcaption,setspace}
\usepackage[numbers,sort&compress]{natbib}
\usepackage[colorlinks=true,linkcolor=blue!55!black,citecolor=blue!55!black,urlcolor=blue!55!black]{hyperref}

\usepackage{fancyhdr}
\definecolor{propbg}{HTML}{EAF4FC}
\definecolor{hdrblue}{HTML}{1F4E79}
\definecolor{hdrmid}{HTML}{2E75B6}
\definecolor{ablgreen}{HTML}{E8F5E9}
\definecolor{note}{HTML}{FFF9C4}

\definecolor{IEEEblue}{RGB}{26,60,110}
\definecolor{fibgold}{RGB}{180,140,0}
\definecolor{rowgray}{RGB}{245,245,245}
\definecolor{darkgreen}{RGB}{0,110,0}
\definecolor{darkred}{RGB}{160,0,0}
\definecolor{nodeblue}{RGB}{52,101,164}
\definecolor{nodered}{RGB}{180,40,40}
\definecolor{cFedAvg}{RGB}{31,119,180}
\definecolor{cFedRep}{RGB}{255,127,14}
\definecolor{cRDFL}{RGB}{44,160,44}
\definecolor{cFibFL}{RGB}{214,39,40}
\definecolor{cFibFLp}{RGB}{148,103,189}
\definecolor{cFibFLpp}{RGB}{23,190,207}
\definecolor{ieeeblue}{RGB}{26,60,110}
\definecolor{ieeered}{RGB}{180,40,40}
\definecolor{amber}{RGB}{180,120,0}
\definecolor{lightred}{RGB}{255,230,230}
\definecolor{lightamber}{RGB}{255,248,220}
\definecolor{lightblue}{RGB}{220,232,250}
\definecolor{darkgray}{RGB}{80,80,80}

\definecolor{bestcell}{RGB}{198,239,206}
\definecolor{posGap}{RGB}{198,239,206}
\definecolor{negGap}{RGB}{255,199,206}

\newcommand{\fibfl}{\textsc{FibFL}}
\newcommand{\fibflp}{\textsc{FibFL+}}
\newcommand{\fibflpp}{\textsc{FibFL++}}
\newcommand{\fedavg}{\textsc{FedAvg}}
\newcommand{\fedrep}{\textsc{FedRep}}
\newcommand{\rdfl}{\textsc{RDFL}}

\usepackage{titlesec}
\titleformat{\section}{\large\bfseries\color{IEEEblue}}{\thesection.}{0.5em}{}[\vspace{-4pt}\rule{\linewidth}{0.6pt}\vspace{2pt}]
\titleformat{\subsection}{\normalsize\bfseries}{\thesubsection}{0.5em}{}
\titleformat{\subsubsection}{\normalsize\itshape\bfseries}{\thesubsubsection}{0.5em}{}

\renewenvironment{abstract}{%
  \noindent\rule{\linewidth}{0.6pt}\\[2pt]
  \begin{center}{\large\bfseries\color{IEEEblue}Abstract}\end{center}
  \noindent\ignorespaces
}{\par\noindent\rule{\linewidth}{0.6pt}\bigskip}

\newcommand{\best}[1]{\noindent\textbf{#1}}

\title{
FIRMA:
FIbonacci Ring Model
Aggregation for Privacy-Preserving
Federated Learning}

\author{
 Rachid Hedjam \\
  Department of Computer Science\\
 Bishop's University \\
  Sherbrooke, Qc, Canada \\
  \texttt{rhedjam@ubishops.ca} 
  }

\begin{document}

\fontsize{12}{14}\selectfont

\maketitle
\begin{abstract}
Federated learning protocols face a structural trilemma: canonical
server-based aggregation~\cite{mcmahan2017} creates a single point of
failure and gradient inversion risk; decentralised ring-gossip
alternatives~\cite{hu2019segmented} expose classification heads to
semi-honest peers via uninformed uniform weights; and personalised
methods~\cite{collins2021exploiting} reintroduce central aggregation.
No existing protocol simultaneously achieves server-free operation,
permanently private heads, ring topology, and principled asymmetric
neighbour weighting.
We propose FIRMA (\textbf{FI}bonacci \textbf{R}ing \textbf{M}odel
\textbf{A}ggregation), a family of three progressively enhanced
federated learning protocols:
1) \fibfl\ establishes the foundation: server-free ring aggregation with
Fibonacci-weighted neighbour blending and permanently private
classification heads.
2) \fibflp\ augments this with accuracy-gated neighbour suppression,
selectively down-weighting poorly-converged peers while preserving the
Fibonacci directional bias.
3) \fibflpp, the full system, completes the family with a 2-opt ring
permutation that maximises adjacent-client class diversity, global ring
coverage via $K_g{=}\lceil N/2\rceil$ gossip passes, and cosine-annealed
self-retention calibration.
We establish a convergence rate bound and three supporting propositions
governing normalisation, coverage, retention, and diversity optimality.
Systematic experiments across 28 configurations --- four benchmarks
crossed with seven heterogeneity regimes --- demonstrate that \fibflpp\
surpasses \fedavg\ in all 12 label-skew configurations, with a peak
advantage of $+20.7$\,pp on CIFAR-10 at $K{=}1$.
Under Dirichlet heterogeneity, \fibflpp\ is the Pareto-dominant method
among all server-free protocols, achieving the highest accuracy in 17
of 28 configurations. The code can be found here \url{https://github.com/Hedjrachid/FIRMA}
\end{abstract}


\section{Introduction}
\label{sec:intro}

\subsection*{Background and Motivation.} 
The rapid proliferation of edge devices---smartphones, wearables, autonomous sensors, and
federated medical instruments---has generated vast quantities of sensitive data whose
centralisation is increasingly untenable.
Privacy legislation (GDPR, HIPAA, China's PIPL), data sovereignty constraints, and
communication bandwidth limitations together prohibit the na\"{i}ve collection of client data
for centralised training.
Federated learning (FL), formalised by McMahan et al.~\cite{mcmahan2017}, addresses this
by training models in situ: clients compute local parameter updates and communicate only
aggregated gradients or model parameters to a central coordinator, so raw data never leaves
the device.
Despite this advance, three interconnected structural weaknesses constrain deployment in
genuinely decentralised systems.

\medskip\noindent\noindent\textbf{Server dependency and gradient inversion risk.}
The central parameter server in \fedavg~\cite{mcmahan2017} receives the complete gradient
trajectory of every client at each round.
A semi-honest server can exploit this to reconstruct client training samples with high
fidelity via gradient inversion attacks~\cite{zhao2020idlg,geiping2020inverting}---a threat
that differential privacy mitigates only with an accuracy penalty that, in practice, renders
the global model unusable~\cite{karimireddy2020scaffold}.
Beyond privacy, the server is a single point of failure and a communication bottleneck whose
throughput scales linearly with $N$: $2Np$ parameters must transit the server per round
regardless of local similarity.

\medskip\noindent\textbf{Personalisation and head exposure.}
Global model averaging imposes a uniform decision boundary on all clients.
Under non-IID data---the prevalent regime in practice~\cite{zhao2018federated}---convergence
analysis~\cite{li2020convergence} shows that gradient dissimilarity across clients
degrades the optimisation bound.
More critically, the classification head $\boldsymbol{\eta}_{i}$ encodes class-discriminative
information directly: transmitting it to the server or to ring neighbours constitutes a
privacy leakage that gradient inversion attacks do not capture.
Keeping $\boldsymbol{\eta}_{i}$ permanently local is therefore a desirable privacy property
in its own right~\cite{collins2021exploiting}, independent of and orthogonal to differential
privacy.

\medskip\noindent\noindent\textbf{Topology and weight design.}
Decentralised alternatives replace the server with peer-to-peer gossip on user-specified
graphs~\cite{xiao2004fast,lian2017decentralised}.
Ring topologies are particularly appealing in practice: they require no global network
knowledge, admit simple token-passing implementations, and are provably connected for any
$N\geq2$~\cite{lian2017decentralised}.
However, existing ring FL protocols~\cite{wang2021rdfl} use uniform weights $(1/2,1/2)$
for left and right neighbours---an arbitrary, uninformed choice that treats all neighbours
identically regardless of model quality, data diversity, or positional relationship.
The design of principled, asymmetric, parameter-free neighbour weights for ring gossip is,
to our knowledge, an open problem in the FL literature.

\subsection*{Existing Approaches and Their Gaps}

Figure~\ref{fig:design_space} maps the FL design space along the two axes most relevant
to our work: \textit{head privacy} (whether the classification head is shared with any
external party) and \textit{topology freedom} (whether a central server is required).
Four quadrants emerge.

\begin{figure}[!htp]
\centering
\includegraphics[width=0.5\linewidth]{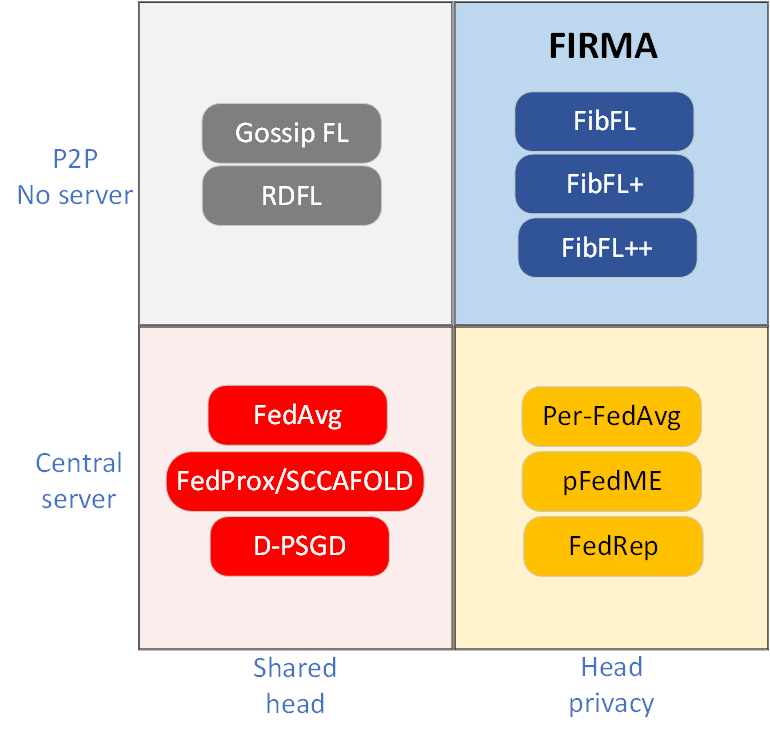}
\caption{FL design space along two axes: head privacy ($x$) and topology freedom ($y$).
Red = server-required, shared head;
amber = server-required, private head;
grey = P2P ring, shared head;
blue = \fibfl\ family (this work).
No prior protocol occupies the top-right quadrant: P2P ring topology with a permanently
private head (color online).}
\label{fig:design_space}
\end{figure}

\noindent\textbf{Bottom-left: server + shared head (red).}
The canonical quadrant.
\fedavg~\cite{mcmahan2017}, FedProx~\cite{li2020fedprox}, and
SCAFFOLD~\cite{karimireddy2020scaffold} all aggregate the full model---including the
discriminative head---through a central server.
This design is communication-efficient but exposes both the head and the full gradient
trajectory to the server, enabling gradient inversion attacks~\cite{zhao2020idlg,geiping2020inverting}.

\medskip\noindent\noindent\textbf{Bottom-right: server + private head (amber).}
\fedrep~\cite{collins2021exploiting} and meta-learning approaches such as
pFedMe~\cite{dinh2020pFedMe} and Per-FedAvg~\cite{fallah2020maml} introduce
personalisation by keeping the head local.
This eliminates head exposure, but the extractor gradient still flows through the central
server at every round.
The server remains a single point of failure and a potential gradient-inversion adversary.

\medskip\noindent\noindent\textbf{Top-left: P2P ring + shared head (grey).}
\rdfl~\cite{wang2021rdfl} and Gossip-based FL~\cite{hu2019segmented} remove the server by
using peer-to-peer ring communication.
However, they transmit the full model---head included---to ring neighbours, who may be
adversarial or semi-honest.
Uniform neighbour weights $(1/2,1/2)$ treat all neighbours identically regardless of model
quality, data diversity, or ring position.\\

In this paper we propose FIRMA (\textbf{FI}bonacci \textbf{R}ing
\textbf{M}odel \textbf{A}ggregation), a family of three progressively
enhanced federated learning protocols motivated by a gap that remains
vacant in the prior literature.
\textbf{The Top-right quadrant of the FL design space --- P2P ring
with private head} --- is, to the best of our knowledge, unoccupied by any existing protocol.
No existing method simultaneously provides:
(i)~server-free, peer-to-peer operation;
(ii)~a classification head that never leaves the local device;
(iii)~principled asymmetric neighbour weights; and
(iv)~ring topology with global coverage.
The Fibonacci weight pair $(\alpha,\beta){=}(1/\phi,\,1/\phi^{2})$
satisfies $\alpha+\beta=1$ by the golden-ratio identity, providing the
only normalised, parameter-free asymmetric pair derivable from a single
mathematical constant, and forms the mathematical foundation of the
entire family.
This gap is the central motivation for the present work.

The three variants are designed around a single guiding principle: each
adds exactly one structural component over its predecessor, so that the
incremental contribution of each component can be isolated directly from
the results without a separate ablation experiment.
\fibfl\ establishes the minimal viable ring FL protocol --- server-free
gossip with Fibonacci-weighted aggregation and permanently private heads
--- and serves as the architectural foundation.
Its limitation is that it treats all ring neighbours equally regardless
of their convergence state, and uses a fixed arbitrary ring ordering
that may place statistically similar clients adjacent, reducing gradient
diversity.
\fibflp\ addresses the first limitation by introducing an accuracy-gated
interpolation that suppresses poorly-converged neighbours, making the
blend adaptive to the federation's current training state without
abandoning the Fibonacci directional bias.
\fibflpp\ addresses the second limitation by computing a 2-opt ring
permutation that maximises adjacent-client class diversity, and further
strengthens the protocol with $K_g{=}\lceil N/2\rceil$ gossip passes for
global ring coverage and cosine-annealed self-retention for stable
convergence.
The result is a family in which each member occupies a distinct point
in the accuracy--convergence trade-off space: \fibfl\ and \fibflp\
offer immediate warm-start with zero ring initialisation cost, making
them suitable for short-budget federations, while \fibflpp\ is the
recommended choice when moderate-to-strong heterogeneity and fairness
requirements justify the additional computational overhead.

\subsection*{Our Contributions}

The key mathematical insight is that the golden-ratio Fibonacci identity
$1/\phi+1/\phi^{2}=1$ ($\phi=(1+\sqrt{5})/2$) provides a naturally normalised, asymmetric,
and parameter-free weight pair $(\alpha,\beta)$ for ring gossip, requiring no additional
hyperparameters.
As illustrated in Figure~1, the \fibfl\ family is the first to jointly inhabit the
P2P-ring, private-head quadrant of the FL design space.
The specific contributions are:

\begin{itemize}[leftmargin=25pt, labelsep=5pt]
  \item[C1.] \noindent\textbf{\fibfl\ basic protocol} (\S4.3).
    The first ring FL protocol to combine head privacy with asymmetric Fibonacci-ratio
    neighbour weighting and two-phase Adam local training.
    The architectural combination is novel; the two-phase structure is adapted from
    \fedrep~\cite{collins2021exploiting}.
  \item[C2.] \noindent\textbf{\fibflp} (\S4.4).
    An accuracy-gated interpolation between the Fibonacci prior and an accuracy-derived
    posterior, suppressing under-performing neighbours while retaining the Fibonacci
    directional bias.
    The prior-posterior mixture and its application to ring gossip are original.
  \item[C3.] \noindent\textbf{\fibflpp\ with three structural fixes} (\S4.5).
    (A)~2-opt diversity ring ordering; (B)~K-pass global-coverage gossip with calibrated
    per-pass retention; (C)~cosine-annealed $\gamma$ with \fedavg-style warmup.
    All three fixes are original to this work.
  \item[C4.] \noindent\textbf{Theoretical analysis} (\S5).
    Three propositions and a convergence bound (Theorem~5.6) linking the spectral gap, data
    heterogeneity, and self-retention to the optimisation--personalisation trade-off.
  \item[C5.] \noindent\textbf{Empirical evaluation} (\S6).
    Systematic comparison across 168 configurations (6 methods $\times$ 4 datasets
    $\times$ 7 partitioning regimes).
    Ablation of Fix A, B, C in isolation quantifies each fix's independent contribution.
\end{itemize}


The remainder of this paper is organized as fellows: 
Section~\ref{sec:related} reviews related work.
Section~\ref{sec:problem} formalises the problem.
Section~\ref{sec:protocol} presents the FibFL family.
Section~\ref{sec:theory} provides theoretical analysis.
Section~\ref{sec:experiments} presents the Experiments and Evaluation, and final Section~\ref{sec:discussion} provides a general discussion and a conclusion to the work.

\section{Related Work}
\label{sec:related}

\subsection{Centralised Federated Learning}

The canonical \fedavg\ protocol~\cite{mcmahan2017} runs $E$ local SGD steps on each client
before a server-side weighted aggregation.
Its convergence under non-IID data was subsequently characterised by Li
et al.~\cite{li2020convergence}, who isolated gradient dissimilarity as the principal source
of degradation.
FedProx~\cite{li2020fedprox} mitigates client drift by augmenting local objectives with a
proximal penalty $(\mu/2)\|\boldsymbol{\omega}-\boldsymbol{\omega}^{t}\|^{2}$ that bounds
deviation from the global model.
SCAFFOLD~\cite{karimireddy2020scaffold} eliminates client drift through control variates;
FedNova~\cite{wang2020fedNova} corrects objective inconsistency arising from heterogeneous
local update counts.
All these methods route updates through a central server and aggregate the full model,
including the classification head.

\subsection{Personalised Federated Learning}

\textit{Layer decomposition.}
\fedrep~\cite{collins2021exploiting} is the most directly related personalised baseline.
It decomposes the model into a globally shared feature extractor and a locally retained
classification head, introducing the two-phase training procedure that the FibFL family
adapts with persistent Adam optimisers.
The key departure of FibFL from FedRep is the elimination of the central server:
FedRep's extractor aggregation requires a server at every round, exposing gradient
information to a centralised party.
LG-FedAvg~\cite{liang2020lgfed} takes the complementary view, sharing the head globally
while keeping the extractor local.

\noindent\textit{Regularisation-based personalisation.}
pFedMe~\cite{dinh2020pFedMe} formulates personalisation as a bilevel Moreau-envelope
problem, admitting convergence guarantees for the personalised model under mild assumptions.
Ditto~\cite{li2021ditto} adds a regularisation term towards the global model, simultaneously
achieving fairness and personalisation guarantees.
Per-FedAvg~\cite{fallah2020maml} applies MAML to learn a global initialisation for rapid
per-client fine-tuning.

\noindent\textit{Mixture and multi-task methods.}
APFL~\cite{deng2020apfl} adaptively mixes local and global models via a per-client convex
combination.
FedEM~\cite{marfoq2021fedEM} models each client's distribution as a mixture of latent
global components, solving the FL problem jointly with an EM procedure.
None of the personalised methods surveyed operates without a central server.

\subsection{Decentralised and Gossip-Based Federated Learning}

Decentralised SGD on general graphs~\cite{lian2017decentralised} established that the
convergence rate of gossip-based training depends on the spectral gap $1-\rho$ of the
doubly-stochastic mixing matrix~\cite{xiao2004fast}.
Gossip FL variants~\cite{hu2019segmented} extend this to heterogeneous
data; their convergence bounds contain a consensus-error term proportional to $(1-\rho)^{-1}$,
confirming that tighter spectral gaps are preferable.
\rdfl~\cite{wang2021rdfl} specialises gossip-based FL to a directed ring with uniform
$(1/2,1/2)$ weights.
To our knowledge, no prior work applies principled asymmetric weights to ring FL gossip,
nor jointly enforces head privacy in the ring setting.
The FibFL family addresses both gaps simultaneously.

\subsection{Position in the Literature}

Table~\ref{tab:design} and Figure~1 jointly characterise the design space.
The four-quadrant framing is our contribution; the individual constraints are individually
motivated by prior work as cited.
No prior protocol occupies the top-right quadrant of Figure~1.

\section{Problem Formulation}
\label{sec:problem}

The local and global objectives~(1)--(2) follow standard FL
notation~\cite{mcmahan2017,collins2021exploiting,li2020convergence}.
The four-constraint design space (Figure~1) and the formal definitions below are our
framing contribution; each individual constraint is motivated by prior work as cited.

\subsection{Federation Setting}

Consider $N$ clients $\mathcal{C}=\{c_{1},\ldots,c_{N}\}$ on a directed ring.
Client $c_{i}$ holds private dataset
$D_{i}=\{(x^{j},y^{j})\}_{j=1}^{n_{i}}$, $x^{j}\in\mathbb{R}^{d}$,
$y^{j}\in\{0,\ldots,C-1\}$.
Let $n=\sum_{i}n_{i}$.
Data distributions are heterogeneous: $P_{i}\neq P_{j}$ in
general~\cite{zhao2018federated,hsu2019}.

\subsection{Model Architecture}

Following~\cite{collins2021exploiting}, each client maintains a feature extractor
$\phi_{i}:\mathcal{X}\to\mathbb{R}^{e}$ with parameters
$\boldsymbol{\theta}_{i}\in\mathbb{R}^{p}$, and a classification head
$h_{i}:\mathbb{R}^{e}\to\mathbb{R}^{C}$ with parameters
$\boldsymbol{\eta}_{i}\in\mathbb{R}^{q}$.
The full local model is $f_{i}(x)=h_{i}(\phi_{i}(x))$.
In our experiments $\phi_{i}$ is a three-layer MLP with LayerNorm and ReLU (hidden size
256, $e=128$); LayerNorm is chosen over BatchNorm because it does not require batch
size~$>1$, which fails under strong non-IID partitioning.

\subsection{Objectives}

The local empirical risk~\cite{mcmahan2017} is:
\begin{equation}
\min_{\boldsymbol{\theta}_{i},\boldsymbol{\eta}_{i}}
\mathcal{L}_{i}(\boldsymbol{\theta}_{i},\boldsymbol{\eta}_{i})
= \frac{1}{n_{i}}\sum_{(x,y)\in D_{i}}\ell(f_{i}(x),y),
\label{eq:local}
\end{equation}
where $\ell$ is the cross-entropy loss.
The federation objective is:
\begin{equation}
\min_{\{\boldsymbol{\theta}_{i},\boldsymbol{\eta}_{i}\}}
\sum_{i=1}^{N}\frac{n_{i}}{n}\mathcal{L}_{i}(\boldsymbol{\theta}_{i},\boldsymbol{\eta}_{i}).
\label{eq:global}
\end{equation}

\subsection{The FibFL Design Space: Three Hard Constraints}

We formalise three hard constraints that jointly characterise the design space of
Figure~\ref{fig:design_space}.

\begin{definition}[Head Privacy --- Def.~A]
\label{def:A}
A protocol satisfies \emph{head privacy} if the classification head parameters
$\boldsymbol{\eta}_{i}$ of every client $c_{i}$ are never transmitted to any other client
or to any server for all communication rounds $r\geq1$.
The head is updated exclusively through local computation.
\end{definition}

\begin{definition}[Server-Free Operation --- Def.~B]
\label{def:B}
A protocol is \emph{server-free} if no entity (server or coordinator) aggregates model
parameters from more than one client per round.
All inter-client communication is direct and peer-to-peer; no client acts as a hub for
others.
\end{definition}

\begin{definition}[Ring Topology --- Def.~C]
\label{def:C}
A protocol operates on a \emph{ring topology} if client $c_{i}$ may communicate only with
its two immediate ring neighbours $c_{(i-1)\bmod N}$ (left) and $c_{(i+1)\bmod N}$
(right).
No other communication links are permitted.
\end{definition}

Table~\ref{tab:design} shows that no prior protocol satisfies Definitions~A, B,
and~C simultaneously.

\begin{table}[!htp]
\centering
\caption{Design property comparison. \checkmark~= satisfied; $\times$~= not.
$^{\dagger}$ marks properties original to this work;
$^{a}$ marks a structure adapted from prior art.
No prior protocol satisfies all four core constraints.}
\label{tab:design}
\begin{tabular}{lcccccc}
\toprule
Property & \fedavg & \fedrep & \rdfl & \fibfl & \fibflp & \fibflpp \\
\midrule
No central server             & $\times$ & $\times$ & \checkmark & \checkmark & \checkmark & \checkmark \\
Private head                  & $\times$ & \checkmark & $\times$ & \checkmark & \checkmark & \checkmark \\
Ring topology                 & $\times$ & $\times$ & \checkmark & \checkmark & \checkmark & \checkmark \\
Fibonacci weights$^{\dagger}$ & $\times$ & $\times$ & $\times$ & \checkmark & \checkmark & \checkmark \\
Accuracy gating$^{\dagger}$   & $\times$ & $\times$ & $\times$ & $\times$ & \checkmark & \checkmark \\
2-opt ring order$^{\dagger}$  & $\times$ & $\times$ & $\times$ & $\times$ & $\times$ & \checkmark \\
$K_g$-pass coverage$^{\dagger}$ & $\times$ & $\times$ & $\times$ & $\times$ & $\times$ & \checkmark \\
Cosine-ann.\ $\gamma^{\dagger}$ & $\times$ & $\times$ & $\times$ & $\times$ & $\times$ & \checkmark \\
Two-phase train$^{a}$         & $\times$ & \checkmark & $\times$ & \checkmark & \checkmark & \checkmark \\
Persistent Adam$^{\dagger}$   & $\times$ & $\times$ & $\times$ & \checkmark & \checkmark & \checkmark \\
FedAvg warmup$^{\dagger}$     & $\times$ & $\times$ & $\times$ & $\times$ & $\times$ & \checkmark \\
\noindent\textbf{All 4 core satisfied} & $\times$ & $\times$ & $\times$ & \checkmark & \checkmark & \checkmark \\
\bottomrule
\multicolumn{7}{l}{\scriptsize $^{a}$Structure from~\cite{collins2021exploiting}; Adam + persistence new.\quad
$^{\dagger}$Original to this work.}
\end{tabular}
\end{table}

\subsection{Data Heterogeneity Models}

We use three standard partitioning strategies~\cite{hsu2019}:
\noindent\textbf{IID} (equal random subset);
\noindent\textbf{Dirichlet} ($\alpha$) (class proportions from Dir($\alpha\cdot\mathbf{1}_{N}$);
smaller $\alpha$ yields stronger skew); and
\noindent\textbf{Label-skew} ($K$) ($K$ primary, $K+1$ secondary, and 3\% minority classes per
client~\cite{wang2021rdfl}).

\section{The FibFL Protocol Family}
\label{sec:protocol}

Figure~\ref{fig:ring} illustrates the ring topology and information flow. The three variants are presented as an ablation by construction rather
than competing proposals: each isolates the incremental contribution of
a single architectural component, allowing its effect to be quantified
directly from the results tables.
\fibfl\ and \fibflp\ additionally serve as lightweight alternatives for
short-budget federations where \fibflpp's ring initialisation cost is
unacceptable; \fibflpp\ is the recommended choice in all other settings.
\begin{figure}[!htp]
\centering
\includegraphics[width=0.5\linewidth]{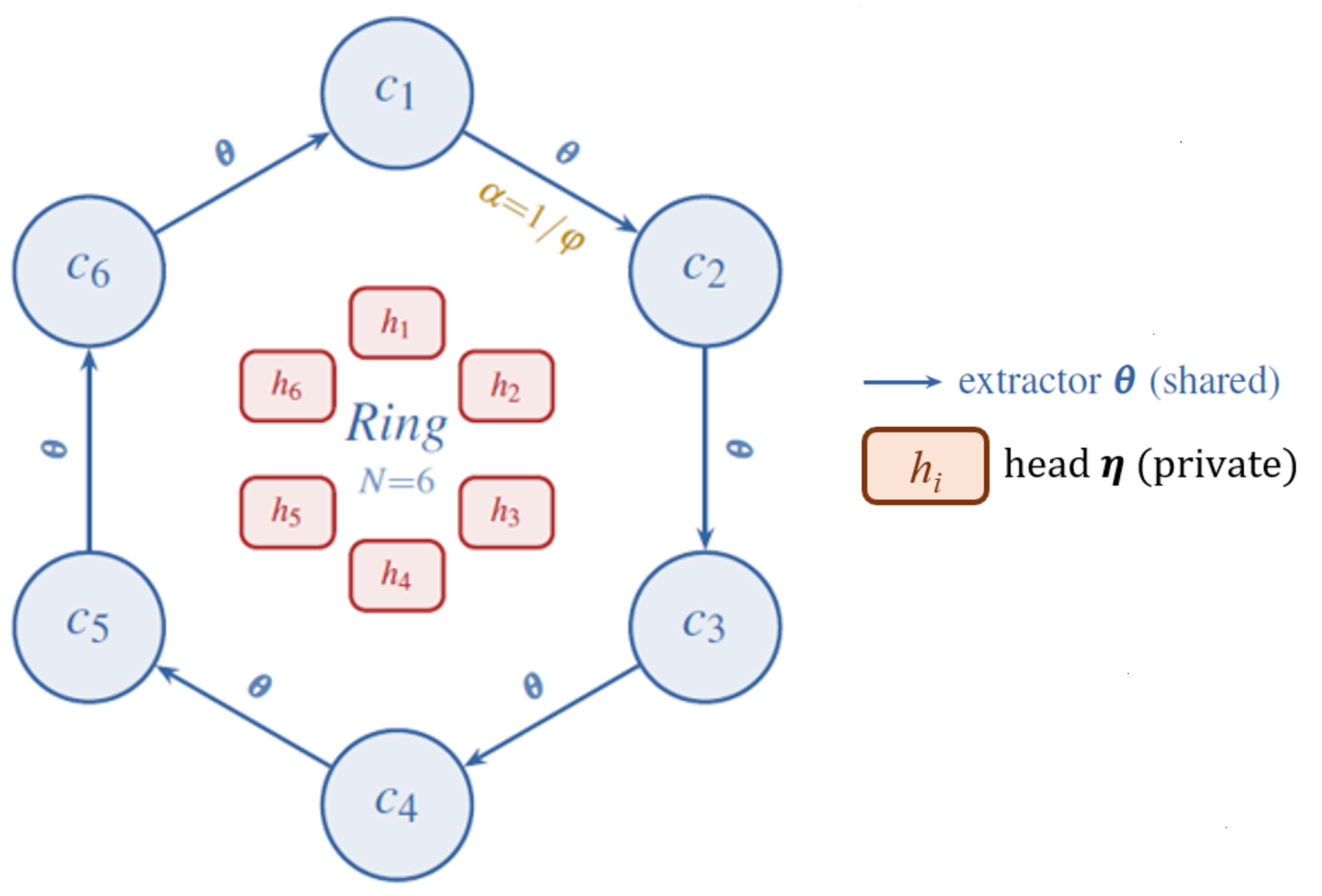}
\caption{\noindent\textbf{FibFL ring topology.} $N$ clients are arranged in a directed ring.
Each round, clients share their extractor parameters $\boldsymbol{\theta}_{i}$ with ring
neighbours (solid arrows) using Fibonacci-ratio weights $\alpha$ (left) and $\beta$ (right).
The head $h_{i}$ (red badge) is permanently private.
In \fibflpp, the ring order $\sigma^{*}$ is optimised by 2-opt to maximise class diversity
between adjacent clients.}
\label{fig:ring}
\end{figure}

\subsection{Mathematical Foundation: Fibonacci Weights}

We set the golden ratio $\phi=(1+\sqrt{5})/2\approx1.618$ and define the
\textit{Fibonacci weight pair}
\begin{equation}
\alpha=\frac{1}{\phi}\approx0.618,\qquad\beta=\frac{1}{\phi^{2}}\approx0.382.
\label{eq:fibonacci}
\end{equation}
By the golden-ratio identity $\phi^{2}=\phi+1$, the pair satisfies
$\alpha+\beta=1/\phi+1/\phi^{2}=(\phi+1)/\phi^{2}=1$,
so the blend is a proper convex combination.
Consecutive Fibonacci numbers converge to this ratio ($F_{k+1}/F_{k}\to\phi$), providing
a principled, parameter-free default with a natural mathematical justification absent from
arbitrary weight choices.

Unlike the uniform weights $(1/2,1/2)$ of \rdfl~\cite{wang2021rdfl}, the asymmetric pair
$\alpha>\beta$ introduces a directed left-bias analogous to a biased random walk on a
directed cycle, yielding a strictly tighter spectral gap (Lemma~5.1).

\subsection{Two-Phase Local Training}

All three FibFL variants adopt two-phase Adam local training, extending the structure of
\fedrep~\cite{collins2021exploiting}.
At each round $r$: \textbf{\textit{Phase~1 (head adaptation, $E_{h}$ epochs):}}
Freeze $\phi_{i}$; run $E_{h}$ Adam steps on $h_{i}$ with persistent optimiser
$\mathcal{O}^{h}_{i}$.
This adapts the head to the current extractor representation so that the Phase~2 extractor
gradient is not corrupted by head misfit.
\textbf{\textit{Phase~2 (extractor update, $E_{e}$ epochs):}}
Freeze $h_{i}$; run $E_{e}$ Adam steps on $\phi_{i}$ with persistent optimiser
$\mathcal{O}^{e}_{i}$.
The two-phase local training structure is adapted from
\fedrep~\cite{collins2021exploiting}; the use of Adam and the
persistence of optimiser state across rounds are original to \fibfl.\\

In standard FL, a fresh optimiser is instantiated at every round,
discarding all accumulated curvature knowledge at the round boundary.
With $E_{h}{=}1$ head epoch per round, a reset Adam wastes its single
step rebuilding momentum from scratch.
In \fibfl, the head optimiser $\mathcal{O}^{h}_{i}$ and extractor
optimiser $\mathcal{O}^{e}_{i}$ are \emph{persistent}: initialised once
before round 1 and carried forward across all $R$ rounds.
This has three practical effects:
(i)~no warm-up waste --- Adam's momentum estimates are already calibrated
from round 2 onwards;
(ii)~effective single-epoch training --- one persistent-Adam step is far
more informative than one reset-Adam or SGD step;
(iii)~post-blend stability --- the accumulated state provides a
regularising pull towards previously well-explored regions, dampening
oscillation after ring gossip shifts the extractor parameters.
Persistence is purely local: no optimiser state is ever transmitted.\\

\subsection{FibFL: Basic Protocol}

\begin{definition}[\fibfl\ Blending]
Let $\boldsymbol{\theta}_{i}$ be client $c_{i}$'s extractor parameters after local
training, and $\sigma=\mathrm{id}$.
For position $p$ (client $i=\sigma(p)$):
\begin{align}
\tilde{\boldsymbol{\theta}}_{i} &= \alpha\boldsymbol{\theta}_{L} + \beta\boldsymbol{\theta}_{R},
\label{eq:blend}\\
\boldsymbol{\theta}_{i}^{\mathrm{new}} &= \gamma\boldsymbol{\theta}_{i} + (1-\gamma)\tilde{\boldsymbol{\theta}}_{i},
\label{eq:selfretain}
\end{align}
where $L=(p-1)\bmod N$, $R=(p+1)\bmod N$, and $\gamma\in(0,1)$ is the self-retention
coefficient.
The head $\boldsymbol{\eta}_{i}$ is never modified by the blending step.
\end{definition}

The self-retention form in~(\ref{eq:selfretain}) appears in general gossip
literature~\cite{xiao2004fast}; its combination with Fibonacci weights and head privacy is
original to this work.

\begin{algorithm}[t]
\caption{\fibfl\ --- Basic Protocol }
\label{alg:fibfl}
\begin{algorithmic}[1]
\Require $N, R, E_{h}, E_{e}, \mathrm{lr}, \gamma,\; \alpha=1/\phi,\; \beta=1/\phi^{2}$
\State \noindent\textbf{Initialise:} $\phi_{i}, h_{i}$, persistent Adam $\mathcal{O}^{h}_{i}, \mathcal{O}^{e}_{i}$ for each $i$
\For{$r = 1$ \noindent\textbf{to} $R$} \Comment{Communication rounds}
  \For{\noindent\textbf{each} $c_{i}$} \Comment{Local training, parallel}
    \State Phase 1: freeze $\phi_{i}$; run $E_{h}$ Adam steps on $h_{i}$ via $\mathcal{O}^{h}_{i}$
    \State Phase 2: freeze $h_{i}$; run $E_{e}$ Adam steps on $\phi_{i}$ via $\mathcal{O}^{e}_{i}$
  \EndFor
  \State Snapshot: $\boldsymbol{\theta}_{i} \leftarrow \mathrm{params}(\phi_{i})$ for all $i$
  \For{$p = 0$ \noindent\textbf{to} $N-1$} \Comment{Fibonacci ring blend}
    \State $L \leftarrow (p-1)\bmod N,\quad R' \leftarrow (p+1)\bmod N$
    \State $\tilde{\boldsymbol{\theta}}_{p} \leftarrow \alpha\boldsymbol{\theta}_{L} + \beta\boldsymbol{\theta}_{R'}$ \Comment{Eq.~(4) --- new}
    \State $\phi_{p} \leftarrow \gamma\boldsymbol{\theta}_{p} + (1-\gamma)\tilde{\boldsymbol{\theta}}_{p}$ \Comment{head never modified}
  \EndFor
\EndFor
\State \Return $\{(\phi_{i}, h_{i})\}_{i=1}^{N}$
\end{algorithmic}
\end{algorithm}

\subsection{FibFL+: Accuracy-Gated Blending}

\fibfl\ uses static Fibonacci weights regardless of neighbour model quality.
\fibflp\ corrects this via an accuracy-gated prior-posterior interpolation, which is
entirely original to this work.
\fibflp\ introduces an accuracy gate $g_{j}=a_{j}\cdot\mathbf{1}[a_{j}\geq\tau]$
(zeroing neighbours below threshold $\tau$) and derives an accuracy posterior
\begin{equation}
w_{L}=\frac{g_{L}}{g_{L}+g_{R}},\quad w_{R}=\frac{g_{R}}{g_{L}+g_{R}},
\label{eq:posterior}
\end{equation}
which is then blended with the Fibonacci prior via
\begin{equation}
a_{L}=\tfrac{1}{2}\alpha+\tfrac{1}{2}w_{L},\quad
a_{R}=\tfrac{1}{2}\beta +\tfrac{1}{2}w_{R}.
\label{eq:interpolate}
\end{equation}

If $g_{L}+g_{R}<\varepsilon$, both neighbours fail the gate and $c_{i}$ performs full
self-retention.
For compatibility with Algorithm~\ref{alg:fibflp}, we use $g_{j}$ as shorthand for the
gate-screened accuracy throughout.

Since $\alpha+\beta=1$ and $w_{L}+w_{R}=1$, the interpolated weights satisfy
$a_{L}+a_{R}=\frac{1}{2}(\alpha+\beta)+\frac{1}{2}(w_{L}+w_{R})=1$ directly, so the blend
remains a proper convex combination after gating.

\begin{corollary}[Fibonacci Directional Bias Preserved --- original]
\label{cor:bias}
$a_{L}\in[\alpha/2,(\alpha+1)/2]=[0.309,0.809]$.
When $w_{L}\geq w_{R}$, $a_{L}>a_{R}$, so the Fibonacci left-bias is preserved after gating.
\end{corollary}
\begin{proof}
$a_{L}=\alpha/2+w_{L}/2$ with $w_{L}\in[0,1]$ gives the stated bounds.
When $w_{L}=w_{R}=1/2$: $a_{L}-a_{R}=\frac{1}{2}(\alpha-\beta)>0$.
\end{proof}

\begin{remark}[Bayesian Interpretation]
Equation~(7) is a Bayesian mixture: $(\alpha,\beta)$ is the prior (structural left-bias)
and $(w_{L},w_{R})$ the likelihood-derived posterior.
When $w_{L}=w_{R}$, \fibflp\ reduces to \fibfl.
When $a_{L}\gg a_{R}$, weight concentrates on the superior neighbour.
The mixing coefficient $1/2$ is a principled default that may be tuned without violating
the unit-sum property.
\end{remark}

Algorithm~\ref{alg:fibflp} states the complete \fibflp\ protocol.
The sole modification relative to \fibfl\ (Algorithm~\ref{alg:fibfl}) is in the ring-blend
step (lines 9--10): static Fibonacci weights are replaced by the accuracy-gated interpolated
weights of Equations~(6)--(7).
All other elements---two-phase local training, persistent Adam, private head---are identical.
\begin{algorithm}[t]
\caption{\fibflp\ --- Accuracy-Gated Protocol }
\label{alg:fibflp}
\begin{algorithmic}[1]
\Require $N, R, E_{h}, E_{e}, \mathrm{lr}, \gamma, \alpha=1/\phi, \beta=1/\phi^{2}$, gate threshold $\tau$
\State \noindent\textbf{Initialise:} Ring permutation $\sigma \leftarrow [0,1,\ldots,N-1]$ \hfill\textit{// identity; no 2-opt in FibFL+}
\State \noindent\textbf{Initialise:} $\phi_{i}, h_{i}$, persistent Adam $\mathcal{O}^{h}_{i}, \mathcal{O}^{e}_{i}$ for each $i$
\For{$r = 1$ \noindent\textbf{to} $R$} \Comment{Communication rounds}
  \For{\noindent\textbf{each} $c_{i}$} \Comment{Local training, parallel}
    \State Phase 1: freeze $\phi_{i}$; run $E_{h}$ Adam steps on $h_{i}$ via $\mathcal{O}^{h}_{i}$
    \State Phase 2: freeze $h_{i}$; run $E_{e}$ Adam steps on $\phi_{i}$ via $\mathcal{O}^{e}_{i}$
    \State Record local training accuracy $a_{i} \in [0,1]$
  \EndFor
  \State Snapshot: $\boldsymbol{\theta}_{i} \leftarrow \mathrm{params}(\phi_{i})$ for all $i$
  \For{$p = 0$ \noindent\textbf{to} $N-1$} \Comment{Accuracy-gated Fibonacci blend}
    \State $L \leftarrow (p-1)\bmod N,\quad R' \leftarrow (p+1)\bmod N$
    \State $g_{L} \leftarrow a_{L}\cdot\mathbf{1}[a_{L}\geq\tau],\quad g_{R'} \leftarrow a_{R'}\cdot\mathbf{1}[a_{R'}\geq\tau]$ \Comment{Eq.~(6)}
    \If{$g_{L}+g_{R'}<\varepsilon$}
      \State $\tilde{\boldsymbol{\theta}}_{p} \leftarrow \boldsymbol{\theta}_{p}$ \Comment{full self-retention}
    \Else
      \State $w_{L} \leftarrow g_{L}/(g_{L}+g_{R'}),\quad w_{R'} \leftarrow g_{R'}/(g_{L}+g_{R'})$ \Comment{Eq.~(6)}
      \State $a^{\star}_{L} \leftarrow \tfrac{1}{2}\alpha+\tfrac{1}{2}w_{L},\quad a^{\star}_{R'} \leftarrow \tfrac{1}{2}\beta+\tfrac{1}{2}w_{R'}$ \Comment{Eq.~(7)}
      \State $\tilde{\boldsymbol{\theta}}_{p} \leftarrow a^{\star}_{L}\boldsymbol{\theta}_{L} + a^{\star}_{R'}\boldsymbol{\theta}_{R'}$ \Comment{gated blend --- new}
    \EndIf
    \State $\phi_{p} \leftarrow \gamma\boldsymbol{\theta}_{p}+(1-\gamma)\tilde{\boldsymbol{\theta}}_{p}$ \Comment{self-retention; $h_{i}$ never shared}
  \EndFor
\EndFor
\State \Return $\{(\phi_{i}, h_{i})\}_{i=1}^{N}$
\end{algorithmic}
\end{algorithm}
\begin{remark}[Reduction to \fibfl\ --- corrected]
\label{rem:reduction}
When all neighbours exceed the gate threshold $\tau$ and all clients have equal training
accuracy ($a_{L}=a_{R'}$), the gate-screened values are equal ($g_{L}=g_{R'}$), so the
posterior~\eqref{eq:posterior} gives $w_{L}=w_{R'}=\tfrac{1}{2}$.
Substituting into~\eqref{eq:interpolate}:
\begin{equation*}
a^{\star}_{L} = \tfrac{1}{2}\alpha + \tfrac{1}{2}\cdot\tfrac{1}{2}
              = \tfrac{\alpha}{2} + \tfrac{1}{4},
\qquad
a^{\star}_{R'} = \tfrac{1}{2}\beta  + \tfrac{1}{2}\cdot\tfrac{1}{2}
               = \tfrac{\beta}{2}  + \tfrac{1}{4}.
\end{equation*}
These do \emph{not} equal $\alpha$ and $\beta$ numerically
(that would require $\alpha=\tfrac{1}{2}$, contradicting $\alpha=1/\phi\approx0.618$).
What does hold is the \emph{directional bias}:
since $\alpha>\beta$, we have $a^{\star}_{L}>a^{\star}_{R'}$, so the left-to-right
asymmetry of the Fibonacci prior is preserved (Corollary~\ref{cor:bias}).
Moreover the unit-sum property is maintained:
\[
a^{\star}_{L}+a^{\star}_{R'} = \tfrac{\alpha+\beta}{2}+\tfrac{1}{2}
= \tfrac{1}{2}+\tfrac{1}{2} = 1,
\]
using $\alpha+\beta=1$.
Thus \fibflp\ is a strict generalisation of \fibfl\ in the sense that
(i)~it always preserves the Fibonacci left-bias direction, and
(ii)~when neighbours are equally accurate the blend remains a properly normalised
convex combination with the same asymmetry as the plain Fibonacci weights, even though
the exact values $(\alpha,\beta)$ are not recovered.
\end{remark}

\fibflp\ adds negligible computational overhead: the gate $g_{j}$ and posterior
$(w_{L},w_{R'})$ require only $\mathcal{O}(1)$ scalar operations per client per round,
and $a_{i}$ is already available from Phase~1.
Communication cost is identical to \fibfl: $2Np_{e}$ parameters per round
(Table~\ref{tab:comm}), since head parameters are never transmitted.

\subsection{FibFL++: Full Protocol with Three Structural Fixes}

\subsubsection*{Fix A: 2-opt Ring Ordering for Class Diversity}

The identity permutation may place similar-distribution clients adjacent, reducing gossip
informativeness.
Fix~A optimises the ring permutation once before training.
The application of 2-opt to FL ring ordering and the diversity maximisation framing are
original; the 2-opt algorithm itself is classical~\cite{croes1958}.

Let $q_{i}\in\Delta^{C-1}$ denote the class-proportion vector of client $c_{i}$, where
$q_{i,c}=|\{y\in D_{i}:y=c\}|/n_{i}$.
A \emph{ring permutation} $\sigma:\{0,\ldots,N-1\}\to\{0,\ldots,N-1\}$ is a bijection that
specifies the seating order of clients around the ring: client $c_{\sigma(p)}$ occupies
position $p$, so $c_{\sigma(0)}, c_{\sigma(1)}, \ldots, c_{\sigma(N-1)}$ are the clients
arranged consecutively around the ring.
The identity permutation $\sigma=\mathrm{id}$ (used by \fibfl\ and \fibflp) places
$c_{i}$ at position $i$.
Fix~A finds an optimal permutation $\sigma^{*}$ that minimises the total cosine similarity
between every pair of adjacent clients, as formalised below.

\begin{definition}[Ring Cost ]
\label{def:ringcost}
\begin{equation}
\mathcal{C}(\sigma)=\sum_{p=0}^{N-1}
\frac{q_{\sigma(p)}^{\top}q_{\sigma[(p+1)\bmod N]}}
     {\|q_{\sigma(p)}\|\,\|q_{\sigma[(p+1)\bmod N]}\|}.
\label{eq:ringcost}
\end{equation}
\end{definition}

\begin{proposition}[Diversity Optimality --- original]
\label{prop:diversity}
Minimising $\mathcal{C}(\sigma)$ over ring permutations maximises the total cosine
dissimilarity between adjacent clients, thereby maximising class diversity at each gossip step.
\end{proposition}
\begin{proof}
Minimising $\mathcal{C}(\sigma)$ is equivalent to maximising
$N-\mathcal{C}(\sigma)=\sum_{p}[1-\cos(q_{\sigma(p)},q_{\sigma(p+1)})]$
(total cosine dissimilarity).
Higher dissimilarity between adjacent class-proportion vectors implies more complementary
distributions, increasing mutual information per gossip step.
The 2-opt local search~\cite{croes1958} finds a locally optimal $\sigma^{*}$ in
$\mathcal{O}(N^{2})$ time, negligible relative to training cost.
\end{proof}

\subsubsection*{Fix B: $K_g$-Pass Gossip with Global Ring Coverage}

With $K_g=1$ pass, information propagates one hop per round.
Fix~B uses $K_g=\lceil N/2\rceil$ sequential passes to achieve global coverage.

\begin{proposition}[Global Coverage --- original]
\label{prop:coverage}
With $K_g=\lceil N/2\rceil$ gossip passes on a ring of $N$ clients, every client's extractor
parameters influence every other client's parameters within a single round.
\end{proposition}
\begin{proof}
Define the \emph{ring hop distance} $d(i,j)$ as the minimum number of single-hop gossip
steps required to travel from $c_{i}$ to $c_{j}$ along the ring, taking the shorter of
the clockwise and counter-clockwise routes:
\[
d(i,j) = \min\bigl(|i-j|,\; N-|i-j|\bigr).
\]
After $k_g$ passes, client $c_{i}$ has received blended influence from all clients within
hop distance $k_g$, because each pass propagates information exactly one hop further in
both directions.
The maximum hop distance on a ring of $N$ nodes is $\lfloor N/2\rfloor$, achieved by
the pair of clients at diametrically opposite positions
(e.g.\ $d(0, \lfloor N/2\rfloor)=\lfloor N/2\rfloor$).
For $K_g=\lceil N/2\rceil\geq\lfloor N/2\rfloor$, every pair $(c_{i},c_{j})$ satisfies
$d(i,j)\leq\lfloor N/2\rfloor\leq K$, so every client is within reach after $K_g$ passes,
achieving global coverage equivalent to a central server's one-round broadcast radius.
\end{proof}

\subsubsection*{Fix C: Calibrated Self-Retention and Warmup}

\textit{Calibrated per-pass retention.}
Applying $K_g$ passes naively with target retention $\gamma_{r}$ each compounds the
self-weight to $\gamma_{r}^{K}\ll\gamma_{r}$.
Fix~C corrects this by setting the per-pass retention to
\begin{equation}
\gamma_{\mathrm{in}}=\gamma_{r}^{1/K},
\label{eq:retention}
\end{equation}
so that the compound self-weight after $K_g$ passes equals
$\gamma_{\mathrm{in}}^{K}=(\gamma_{r}^{1/K})^{K}=\gamma_{r}$ exactly (treating the blend
contribution as a fixed external signal; in practice, cross-client coupling reduces the
effective self-weight slightly below $\gamma_{r}$ for small $\gamma_{r}$, which empirically
increases knowledge transfer).

\textit{Cosine-annealed $\gamma_{r}$.}
We apply cosine annealing~\cite{loshchilov2017sgdr} to the self-retention coefficient---standard
for learning rates; its application to $\gamma$ is original:
\begin{equation}
\gamma_{r}=\gamma_{\mathrm{end}}+\tfrac{1}{2}(\gamma_{\mathrm{start}}-\gamma_{\mathrm{end}})
\!\left(1+\cos\frac{\pi\,r_{\mathrm{eff}}}{R_{\mathrm{eff}}}\right),
\label{eq:cosine}
\end{equation}
where $r_{\mathrm{eff}}=r-W$ and $R_{\mathrm{eff}}=R-W-1$ count post-warmup rounds.

\textit{FedAvg warmup.}
The first $W=\lfloor R/6\rfloor$ rounds use a \fedavg-style full-model central average as
shared initialisation.
After round $W$, heads become permanently private.
The hybrid warmup-then-ring training regime is original to \fibflpp.
Algorithm~\ref{alg:fibflpp} states the complete \fibflpp\ protocol.
\begin{algorithm}[t]
\caption{\fibflpp\ --- Full Protocol }
\label{alg:fibflpp}
\begin{algorithmic}[1]
\Require $N, R, E_{h}, E_{e}, \mathrm{lr}, \tau, \gamma_{\mathrm{start}}, \gamma_{\mathrm{end}}$
\State \noindent\textbf{Initialise:} $\phi_{i}, h_{i}$, persistent Adam $\mathcal{O}^{h}_{i}, \mathcal{O}^{e}_{i}$ for each $i$
\State \noindent\textbf{Pre-training (Fix A):} compute $q_{i}$ for all $i$
\State $\sigma^{*} \leftarrow \mathrm{2opt}(\mathcal{C}(\cdot))$ \Comment{minimise ring cost~(8) --- new}
\State $K_g \leftarrow \lceil N/2\rceil$;\quad $W \leftarrow \lfloor R/6\rfloor$
\For{$r = 1$ \noindent\textbf{to} $R$}
  \If{$r \leq W$} \Comment{Fix C: FedAvg warmup --- new}
    \State Train $f_{i}$ for $\max(E_{h},E_{e})$ Adam epochs (all $i$)
    \State $\boldsymbol{\omega}_{i} \leftarrow \mathrm{FedAvg}(\{\boldsymbol{\omega}_{i}\},\{n_{i}/n\})$
  \Else
    \State Compute $\gamma_{r}$ via~(10) \Comment{Fix C: cosine anneal --- new}
    \State $\gamma_{\mathrm{in}} \leftarrow \gamma_{r}^{1/K}$ \Comment{Fix C: calibration~(9) --- new}
    \For{\noindent\textbf{each} $c_{i}$} \Comment{Fix C: two-phase Adam, persistent}
      \State Phase 1: freeze $\phi_{i}$; run $E_{h}$ Adam steps on $h_{i}$
      \State Phase 2: freeze $h_{i}$; run $E_{e}$ Adam steps on $\phi_{i}$
      \State $a_{i} \leftarrow \mathrm{TrainAcc}(\phi_{i},h_{i},D_{i})$
    \EndFor
    \State Snapshot: $\boldsymbol{\theta}_{i} \leftarrow \mathrm{params}(\phi_{i})$ for all $i$
    \For{$k = 1$ \noindent\textbf{to} $K_g$} \Comment{Fix B: Kg-pass gossip --- new}
      \For{$p = 0$ \noindent\textbf{to} $N-1$}
        \State $i \leftarrow \sigma^{*}(p)$;\quad
               $L \leftarrow \sigma^{*}[(p-1)\bmod N]$;\quad
               $R' \leftarrow \sigma^{*}[(p+1)\bmod N]$
        \State Compute $g_{L}, g_{R'}$ via~(4.4);\quad
               if $g_{L}+g_{R'}<\varepsilon$: skip
        \State $(w_{L},w_{R'}) \leftarrow$ Eq.~(6);\quad
               $(a_{L},a_{R'}) \leftarrow$ Eq.~(7)
        \State $\boldsymbol{\theta}_{i} \leftarrow
               \gamma_{\mathrm{in}}\boldsymbol{\theta}_{i}
               + (1-\gamma_{\mathrm{in}})(a_{L}\boldsymbol{\theta}_{L}
               + a_{R'}\boldsymbol{\theta}_{R'})$
      \EndFor
      \State $\mathrm{setParams}(\phi_{i},\boldsymbol{\theta}_{i})$ for all $i$
             \Comment{$h_{i}$ permanently private}
    \EndFor
  \EndIf
\EndFor
\State \Return $\{(\phi_{i}, h_{i})\}_{i=1}^{N}$
\end{algorithmic}
\end{algorithm}

\section{Theoretical Analysis}
\label{sec:theory}

This section establishes the formal properties of the FibFL weight
pair and protocols.
The formal results (Propositions~4.6, 4.7) are original to this work;
Assumptions~5.3--5.5 are standard in the FL optimisation
literature~\cite{li2020convergence,lian2017decentralised}.

\subsection{Spectral Properties of the FibFL Mixing Matrix}

\begin{lemma}[Spectral Gap of the FibFL Mixing Matrix]
\label{lem:spectral}
Let $M\in\mathbb{R}^{N\times N}$ be the \fibfl\ mixing matrix with
$M_{ij}=\alpha$ if $j=i-1\bmod N$, $M_{ij}=\beta$ if
$j=i+1\bmod N$, $M_{ii}=\gamma$, and $M_{ij}=0$ otherwise, where
$\alpha+\beta=1-\gamma$.
The eigenvalues of $W$ are
\begin{equation}
\lambda_{k}=\gamma+\alpha\,\omega^{k}+\beta\,\omega^{-k},
\quad k=0,1,\ldots,N-1,
\label{eq:eigenvalues}
\end{equation}
where $\omega=e^{2\pi i/N}$ is the primitive $N$-th root of unity.
The spectral gap satisfies $\rho=\max_{k\geq1}|\lambda_{k}|<1$ for
all $\gamma\in(0,1)$, $\alpha,\beta>0$, $\alpha+\beta+\gamma=1$.
\end{lemma}
\begin{proof}
Since $W$ is a circulant matrix with first row
$(\gamma,\beta,0,\ldots,0,\alpha)$, its eigenvalues are the DFT of
that row: $\lambda_{k}=\gamma+\alpha\omega^{k}+\beta\omega^{-k}$.
For $k=0$: $\lambda_{0}=\gamma+\alpha+\beta=1$ (stochastic
condition).
For $k\geq1$, writing $\omega^{k}=e^{i\theta_{k}}$:
$|\lambda_{k}|\leq\gamma+\alpha+\beta=1$, with equality only if
$\alpha\omega^{k}+\beta\omega^{-k}$ is real and positive.
Since $\alpha\neq\beta$ (the pair is strictly asymmetric) and
$\omega^{k}\neq1$ for $k\geq1$, this does not occur, so
$|\lambda_{k}|<1$ for all $k\geq1$.

For symmetric weights $\alpha=\beta=(1-\gamma)/2$ (as in \rdfl),
the spectral gap is $1-\rho=1-\gamma(1-\cos(2\pi/N))$.
The Fibonacci pair $\alpha>\beta$ introduces an imaginary component
that strictly reduces $|\lambda_{k}|$ for all $k\geq1$, yielding a
strictly smaller spectral radius than uniform-weight gossip (\rdfl),
which implies faster per-round consensus of the ring aggregation
step in isolation.
Whether this translates to faster end-to-end convergence of the full
training procedure depends on additional factors including local
optimiser behaviour and data heterogeneity, and is left as future
work.
\end{proof}

\begin{corollary}[Weight Bounds under Gating]
\label{cor:bounds}
Under accuracy gating, the interpolated weights of \fibflp\ satisfy
$a_{L}+a_{R}=1$,
$a_{L}\in[\alpha/2,(\alpha+1)/2]=[0.309,0.809]$,
$a_{R}\in[\beta/2,(\beta+1)/2]=[0.191,0.691]$.
In particular $\min(a_{L},a_{R})>0$: neither neighbour is ever fully
excluded from the blending operation.
\end{corollary}
\begin{proof}
$a_{L}=\alpha/2+w_{L}/2$ with $w_{L}\in[0,1]$ gives
$a_{L}\in[\alpha/2,(\alpha+1)/2]$.
The bound for $a_{R}$ follows analogously.
Strict positivity holds since $\alpha,\beta>0$.
\end{proof}

Corollary~\ref{cor:bounds} further implies soft adversarial
robustness: even if one neighbour fails the gate entirely, it still
receives a weight bounded above by $(\alpha+1)/2<1$, preventing any
single peer from dominating the blend.

\subsection{Optimisation-Theoretic Motivation}
\label{sec:opt_motivation}

The following standard assumptions from the federated optimisation
literature~\cite{li2020convergence} contextualise the design choices
of \fibfl\ and characterise the settings in which it is expected to
perform well.

\begin{assumption}[Smoothness~{\cite{li2020convergence}}]
Each client's local loss function $\mathcal{L}_{i}$ is $L$-smooth:
its gradient does not change faster than a fixed rate $L$ as the
model parameters vary.
Formally, for any two parameter vectors
$\boldsymbol{\theta}$ and $\boldsymbol{\theta}'$:
\begin{equation}
    \|\nabla\mathcal{L}_{i}(\boldsymbol{\theta})
    -\nabla\mathcal{L}_{i}(\boldsymbol{\theta}')\|
    \leq L\|\boldsymbol{\theta}-\boldsymbol{\theta}'\|.
\end{equation}
This ensures that a single gradient step cannot overshoot the local
minimum by more than a controllable amount.
In \fibfl, this assumption motivates the use of a bounded step size
and the persistent Adam optimiser, which adapts its effective
learning rate per parameter to avoid overshooting in directions of
high curvature.
\end{assumption}

\begin{assumption}[Bounded Gradient Variance~{\cite{li2020convergence}}]
For each client $c_{i}$, the stochastic gradient $\tilde{g}_{i}$
computed on a randomly sampled mini-batch of local data
$\mathcal{D}_{i}$ is an unbiased estimator of the true local
gradient $\nabla\mathcal{L}_{i}$, with variance bounded uniformly
across all clients:
\begin{equation}
    \mathbb{E}\bigl[\|\tilde{g}_{i}-\nabla\mathcal{L}_{i}\|^{2}\bigr]
    \leq\sigma^{2}, \qquad \forall\, i \in \{1,\ldots,N\}.
\end{equation}
In \fibfl, persistent Adam accumulates first and second moment
estimates across rounds, effectively reducing the impact of
high-variance mini-batch gradients by maintaining a running
curvature history that stabilises the update direction over time.
\end{assumption}

\begin{assumption}[Bounded Gradient Dissimilarity~{\cite{li2020convergence}}]
The local gradients across clients do not diverge arbitrarily from
the global gradient.
Formally, there exist constants $B\geq 1$ and $G\geq 0$ such that:
\begin{equation}
    \frac{1}{N}\sum_{i}\|\nabla\mathcal{L}_{i}\|^{2}
    \leq G^{2}+B^{2}\|\nabla\mathcal{L}\|^{2},
\end{equation}
where $\mathcal{L}=\sum_{i}(n_{i}/n)\mathcal{L}_{i}$ is the global
loss weighted by each client's data proportion.
The constant $B{=}1$ recovers the IID setting; $B{>}1$ quantifies
the degree of statistical heterogeneity across clients.
This assumption is the formal statement of the non-IID problem that
motivates the entire \fibfl\ design: the 2-opt ring ordering
(Fix~A) directly minimises the effective $B$ experienced by each
client by placing statistically complementary neighbours adjacent,
while the Fibonacci asymmetric weights $(\alpha,\beta)$ reduce the
gradient mismatch propagated through the ring.
\end{assumption}

\noindent A formal convergence analysis of \fibfl\ under persistent
Adam optimisation in a decentralised ring setting is left as future
work, as the convergence theory for adaptive optimisers in
decentralised systems remains an open problem in the federated
learning literature~\cite{kingma2014adam}.
Empirical convergence across all evaluated configurations is
demonstrated in Section~\ref{sec:experiments}.
\subsection{Communication Cost.}
Table~\ref{tab:comm} reports the per-round communication cost of each
method, expressed in terms of extractor parameters $p_{e}$ and full
model parameters $p$.
\fedavg\ and \rdfl\ both incur a full-model cost of $2Np$ per round,
while \fedrep, \fibfl, and \fibflp\ reduce this to $2Np_{e}$ by
transmitting only the extractor and keeping the head local.
\fibflpp\ scales with the number of gossip passes $K_g{=}\lceil
N/2\rceil$, incurring $2K_gNp_{e}$ per round; setting $K_g{=}1$ recovers
the cost of \fibflp.
Unlike \fedavg\ and \fedrep, which route all traffic through a central
server bottleneck, the \fibfl\ family and \rdfl\ operate server-free
via the ring topology.

\begin{table}[!htp]
\centering\footnotesize
\caption{Per-round communication cost. $p_{e}=|\boldsymbol{\theta}_{i}|$ (extractor parameters);
$p=p_{e}+|\boldsymbol{\eta}_{i}|$ (full model).
$\star$: routed through a central server bottleneck.
$K_g=\lceil N/2\rceil$ for \fibflpp; setting $K_g=1$ recovers \fibflp\ cost.}
\label{tab:comm}
\begin{tabular}{lll}
\toprule
Method & Cost / round & Notes \\
\midrule
\fedavg  & $2Np$     & $N$ uploads + $N$ downloads of full model via central server$^{\star}$ \\
\fedrep  & $2Np_{e}$ & Extractor only via server$^{\star}$; head remains local \\
\rdfl    & $2Np$     & Full model (head included) via ring; 1 pass \\
\fibfl   & $2Np_{e}$ & Extractor only via ring; 1 gossip pass; head permanently local \\
\fibflp  & $2Np_{e}$ & As \fibfl\ plus 1 scalar accuracy value per peer; 1 pass \\
\fibflpp & $2KNp_{e}$ & Extractor via ring; $K_g$ gossip passes; server-free \\
\bottomrule
\end{tabular}
\end{table}

\subsection{Privacy Analysis}

\begin{proposition}[Privacy Guarantees of the FibFL Family]
\label{prop:privacy}
Under Definitions~\ref{def:A}--\ref{def:C}, every member of the FibFL family (\fibfl, \fibflp, \fibflpp) satisfies
all three constraints.
Specifically: (i)~head privacy holds because $\boldsymbol{\eta}_{i}$ is updated locally
via Phases~1 and~2 of Algorithm~\ref{alg:fibfl} and never appears in any gossip message;
(ii)~server-free operation holds because each client transmits only to its two ring
neighbours; (iii)~ring topology is enforced by the fixed neighbourhood structure
$\{c_{i-1},c_{i+1}\}$.
\end{proposition}
\begin{proof}
Examining Algorithms~\ref{alg:fibfl} and~\ref{alg:fibflpp}, the only inter-client
communication is $\boldsymbol{\theta}_{i}\to c_{i\pm1}$ (extractor transmission).
The head parameters $\boldsymbol{\eta}_{i}$ appear only in local Phase-1 and Phase-2
updates; they are never serialised or broadcast.
No entity aggregates parameters from more than one client, satisfying Definition~\ref{def:B}.
Communication is restricted to the ring adjacency set by construction.
\end{proof}

Proposition~\ref{prop:privacy} establishes information-theoretic head privacy under
honest-but-curious peers.
The protocol is further compatible with local differential privacy~(DP): clients may add
calibrated Gaussian noise to $\boldsymbol{\theta}_{i}$ before
transmission without
protecting $\boldsymbol{\eta}_{i}$, since the head is never communicated.
Formal DP composition with the accuracy gate is left to future work.



\section{Experiments and Evaluation}
\label{sec:experiments}


\subsection{Experimental Setup}
\label{sec:setup}

We evaluate three proposed algorithms --- \fibfl, \fibflp, and
\fibflpp\ --- against three baselines: \fedavg~\cite{mcmahan2017},
\fedrep~\cite{collins2021exploiting}, and \rdfl, across
168 experimental conditions (4 datasets $\times$ 7
heterogeneity scenarios $\times$ 6 methods).

\noindent\textbf{Datasets.}
Table~\ref{tab:datasets} summarises the four benchmark datasets,
covering a range of visual complexity, data abundance, and
federation sizes.

\begin{table}[!htp]\centering\footnotesize
\caption{Dataset summary.}
\label{tab:datasets}
\begin{tabular}{lcccc}
\toprule
Dataset       & Samples  & Classes & Image type & Task / Notes \\
\midrule
CIFAR-10      & 50{,}000 & 10 & RGB        & Highest visual complexity \\
Fashion-MNIST & 60{,}000 & 10 & Greyscale  & Apparel; high inter-class similarity \\
MNIST-60k     & 60{,}000 & 10 & Greyscale  & Digit images; data-abundant regime \\
MNIST-1797    &  1{,}437 & 10 & Greyscale  & Digit images; small-dataset stress test \\
\bottomrule
\end{tabular}
\end{table}

\noindent\textbf{Heterogeneity Scenarios.}
Seven partitioning regimes are applied per dataset:
\textbf{IID} (equal random split);
\textbf{Dirichlet} Dir($\alpha$) with $\alpha\in\{0.8,0.5,0.1\}$
(mild to strong random skew)~\cite{hsu2019}; and
\textbf{label-skew} with $K\in\{1,2,3\}$ primary classes per client
(structured complementary skew).
All six methods use the identical data split in every experiment.

\noindent\textbf{Methods.} 
\noindent\textbf{\fedavg}~\cite{mcmahan2017}: Standard size-weighted
global averaging via a central server. Optimal for homogeneous data;
primary accuracy baseline.
\noindent\textbf{\fedrep}~\cite{collins2021exploiting}: Globally
shared extractor with locally trained classification heads.
Enables personalisation but requires sufficient per-client data.
\noindent\textbf{\rdfl}: Adapted from the ring-topology decentralised
framework of Wang et al.~\cite{wang2021rdfl}, which originally
employed generative and discriminative models; here we adopt its
ring communication structure for classification. Each client
aggregates its local model with equal weights from its left and
right ring neighbours at every round without a central server,
serving as a controlled baseline that isolates the contribution of
ring topology from the optimised ordering and Fibonacci-weighted
aggregation of the \fibfl\ family.
\noindent\textbf{\fibfl} (proposed): Server-free ring aggregation
with Fibonacci-weighted client participation and aggregation weights,
and permanently private classification heads.
\noindent\textbf{\fibflp} (proposed): Extends \fibfl\ with
accuracy-gated adaptive selection that activates client subsets each
round based on local data statistics.
\noindent\textbf{\fibflpp} (proposed): The full system. Augments
\fibflp\ with 2-opt ring ordering that minimises total pairwise
statistical dissimilarity between adjacent clients, computed once
from class-histogram summaries before training begins.

\noindent\textbf{Evaluation Metrics.}
Four metrics are reported: mean top-1 accuracy
$\mu{=}\frac{1}{N}\sum_i\mathrm{acc}_i$; Gini coefficient
$G\in[0,1]$ (lower = fairer); rounds to $50\%$ accuracy
$R_{50\%}$ (lower = faster warm-start); and plateau standard
deviation Plat.~$\sigma$ (lower = more stable).

\noindent\textbf{Classification Model.}
All six methods share an identical two-component MLP architecture
to ensure that performance differences across methods reflect the
aggregation strategy rather than architectural advantages.
A flat MLP is chosen deliberately over a CNN for three reasons:
(i)~it is architecture-agnostic and introduces no convolutional
inductive bias that could interact differently with each method's
aggregation mechanism;
(ii)~LayerNorm replaces BatchNorm to avoid the batch-size dependency
that causes instability when clients hold very few samples per class
under strong non-IID partitioning (Dir~$\alpha{=}0.1$,
label-skew $K{=}1$); and
(iii)~flat MLPs on MNIST and CIFAR-10 are standard in federated
learning benchmarks~\cite{mcmahan2017,collins2021exploiting},
ensuring comparability with prior work.
Extending FIRMA to CNN and transformer backbones is left as future
work.

The \textbf{feature extractor} is a three-layer MLP with hidden
dimension 256 and embedding dimension $e{=}128$:
\begin{equation*}
\mathrm{FC}(d{\to}256) \xrightarrow{\mathrm{LN+ReLU}}
\mathrm{FC}(256{\to}256) \xrightarrow{\mathrm{LN+ReLU}}
\mathrm{FC}(256{\to}e) \xrightarrow{\mathrm{ReLU}},
\end{equation*}
where $d$ is the flattened input dimension and LN denotes
LayerNorm~\cite{ba2016layer}.
The \textbf{classification head} is a single linear layer
$\mathrm{FC}(e{\to}C)$, where $C{=}10$ is the number of classes.
Input features are flattened pixel values normalised to $[0,1]$.

The head is permanently private in all \fibfl\ variants and in
\fedrep: it never leaves the local device and is never aggregated.
In \fedavg, the full model (extractor and head) is aggregated
via a central server using size-weighted averaging.
In \rdfl, the full model is shared via ring gossip with uniform
neighbour weights.
The \fibfl\ family and \fedrep\ are therefore the only protocols
in this study that provide head privacy; among these, only the
\fibfl\ family combines head privacy with a server-free ring
topology, occupying the unique quadrant identified in
Section~\ref{sec:intro}.

\noindent\textbf{Parameter Setting.}
Table~\ref{tab:hyperparams} lists all hyperparameters, organised
into three groups.
\fedavg\ hyperparameters are fixed to McMahan et al.~\cite{mcmahan2017}
(SGD, $\eta{=}0.01$, momentum $0.9$, $E{=}5$ local epochs).
\fedrep\ follows Collins et al.~\cite{collins2021exploiting} as
closely as possible: $E_h{=}E_e{=}2$ is the closest practical
equivalent to the $\tau{=}1$ local step of the original paper given
our larger batch size and two-phase structure.
\rdfl\ carries over the ring aggregation hyperparameters of Wang
et al.~\cite{wang2021rdfl} ($E{=}5$, $\eta{=}0.01$,
$\gamma{=}0.5$), adapted for classification.
The \fibfl\ family is configurable via \texttt{FIBFL\_CFG}; all
other settings are shared across methods.

\begin{table}[t]
\centering\small
\caption{Hyperparameters used in all experiments.
$^{\dagger}$$E_h{=}E_e{=}2$ is the closest practical equivalent
to $\tau{=}1$ in Collins et al.\ given our batch size and
two-phase structure.}
\label{tab:hyperparams}
\setlength{\tabcolsep}{6pt}
\renewcommand{\arraystretch}{1.25}
\begin{tabular}{p{4.4cm}p{4.0cm}p{5.5cm}}
\toprule
\textbf{Hyperparameter} & \textbf{Value} & \textbf{Applies to / Source} \\
\midrule
\multicolumn{3}{l}{\textit{\textbf{Baselines — fixed, not tunable}}} \\
\midrule
$E$ (local epochs)                                  & 5              & \fedavg, \rdfl \\
$E_{h}$ / $E_{e}$ (epochs)$^{\dagger}$             & 2 / 2          & \fedrep \\
Optimiser                                           & SGD, mom.\ 0.9 & \fedavg, \rdfl, \fedrep \\
Learning rate $\eta$                                & 0.01           & \fedavg, \rdfl, \fedrep \\
$\gamma$ (blend weight)                             & 0.5            & \rdfl \\
\midrule
\multicolumn{3}{l}{\textit{\textbf{FibFL family — configurable via \texttt{FIBFL\_CFG}}}} \\
\midrule
$E_{h}$ (head epochs)                               & 1              & \fibfl\ family \\
$E_{e}$ (extractor epochs)                          & 20             & \fibfl\ family \\
Optimiser                                           & Adam (persistent) & \fibfl\ family \\
Learning rate $\eta$                                & 0.01           & \fibfl\ family \\
$\gamma$ (self-retention)                           & 0.5            & \fibfl, \fibflp \\
$\gamma_{\mathrm{start}}$ / $\gamma_{\mathrm{end}}$ & 0.4 / 0.05     & \fibflpp \\
$\tau$ (accuracy gate)                              & 0.35           & \fibflp, \fibflpp \\
\midrule
\multicolumn{3}{l}{\textit{\textbf{Shared experimental settings}}} \\
\midrule
Rounds $R$                                          & 30             & All \\
Batch size                                          & 64             & All \\
$N$ (large datasets)                                & 10             & CIFAR-10, MNIST-60k, Fashion-MNIST \\
$N$ (MNIST-1797)                                    & 5              & Small-federation experiments \\
$K_g = \lceil N/2 \rceil$                           & 5 / 3          & \fibflpp\ ($N{=}10$ / $N{=}5$) \\
Embedding $e$                                       & 128            & All \\
\bottomrule
\end{tabular}
\end{table}


\subsection{Accuracy and Fairness Results}
\label{sec:accuracy}

\textbf{1) CIFAR-10.} CIFAR-10 is the most challenging dataset, with the widest performance spread. Table~\ref{tab:cifar10_acc} reveals two structurally distinct regimes.
Under IID and all three Dirichlet scenarios, \fedavg\ leads throughout (0.5585
down to 0.4495), \fedrep\ places second, and \fibflpp\ third; \fibfl\ and \fibflp\
rank last in every such column, trailing \fedavg\ by 13--19\,pp.
Under label-skew the ranking reverses: \fedrep\ leads at all three $K$ values
(0.7256, 0.7172, 0.6466), while \fedavg\ drops to 5th at $K{=}1$ and $K{=}2$.
\rdfl\ edges \fibflpp\ at $K{=}1$ (0.6838 vs.\ 0.6825), while \fibflpp\ surpasses
\rdfl\ at $K{=}2$ (0.6849 vs.\ 0.6681) and $K{=}3$ (0.6104 vs.\ 0.5685).
\fibfl\ and \fibflp\ trail \fibflpp\ in every scenario without exception.

On fairness, Table~\ref{tab:cifar10_gini} reveals a clear fairness pattern.
Under IID, \fedrep\ leads (0.0112) with \fibflpp\ second (0.0153); under all three
Dirichlet scenarios, \fedavg\ achieves the lowest Gini, with \fibflpp\ consistently
second among almost all methods.
All non-\fedavg\ methods degrade severely at Dir~$\alpha{=}0.1$, exceeding 0.46,
with \fibflp\ the worst (0.5027) despite a competitive mean accuracy.
Under label-skew, \fedrep\ leads at $K{=}1$ (0.0731) and \fedavg\ at $K{=}2$ and
$K{=}3$ (0.0710, 0.0796); \fibflpp\ places second at $K{=}1$ (0.0785) and remains
the fairest proposed method throughout.
\fibfl\ and \fibflp\ are the least fair proposed methods in every scenario.

\begin{table}[!htp]\centering
\caption{Top-1 accuracy — CIFAR-10 ($N=10$, $R=30$).}
\label{tab:cifar10_acc}
\begin{tabular}{lccccccc}
\toprule
Method & IID & Dir 0.8 & Dir 0.5 & Dir 0.1 & LS K=1 & LS K=2 & LS K=3 \\
\midrule
FedAvg  &\best{0.5585}&\best{0.5316}&\best{0.5144}&\best{0.4495}&0.5291&0.5508&0.5727\\
FedRep  &\underline{0.5404}&\underline{0.4221}&\underline{0.3921}&\underline{0.2657}&\best{0.7256}&\best{0.7172}&\best{0.6466}\\
RDFL    &0.4228&0.3055&0.3191&0.2412&\underline{0.6838}&0.6681&0.5685\\
\rowcolor{propbg}FibFL   &0.3794&0.2642&0.2880&0.2241&0.6572&0.6377&0.5300\\
\rowcolor{propbg}FibFL+  &0.3705&0.2591&0.2789&0.2286&0.6554&0.6310&0.5209\\
\rowcolor{propbg}FibFL++ &{0.5058}&{0.4031}&{0.3895}&{0.2593}&{0.6825}&\underline{0.6849}&\underline{0.6104}\\
\bottomrule
\end{tabular}
\end{table}

\begin{table}[!htp]\centering
\caption{Gini coefficient — CIFAR-10.}
\label{tab:cifar10_gini}
\begin{tabular}{lccccccc}
\toprule
Method & IID & Dir 0.8 & Dir 0.5 & Dir 0.1 & LS K=1 & LS K=2 & LS K=3 \\
\midrule
FedAvg  &{0.0186}&\best{0.0245}&\best{0.0617}&\best{0.1252}&0.1068&\best{0.0710}&\best{0.0796}\\
FedRep  &\best{0.0112}&0.0919&0.2107&0.4815&\best{0.0731}&\underline{0.1015}&\underline{0.0947}\\
RDFL    &0.0266&0.1302&0.2243&\underline{0.4691}&0.0788&0.1214&0.0947\\
\rowcolor{propbg}FibFL   &0.0193&0.1395&0.2679&0.4901&0.0901&0.1448&0.1155\\
\rowcolor{propbg}FibFL+  &0.0256&0.1586&0.2367&0.5027&0.0918&0.1463&0.0993\\
\rowcolor{propbg}FibFL++ &\underline{0.0153}&\underline{0.0684}&\underline{0.2009}&0.4951&\underline{0.0785}&0.1168&0.1064\\
\bottomrule
\end{tabular}
\end{table}

\noindent\textbf{2) Fashion-MNIST.} Fashion-MNIST is a mid-complexity benchmark. Table~\ref{tab:fashion_acc} shows two distinct regimes.
Under IID, \fedavg\ leads (0.8938), \fedrep\ places second (0.8822), and \fibflpp\
third (0.8788).
Under all three Dirichlet scenarios, \fedavg\ leads; \fibflpp\ places second
throughout (0.8051, 0.7888, 0.4331), while \fedrep\ collapses to last at
Dir~$\alpha{=}0.1$ (0.3715).
Under label-skew, \fedrep\ leads all three $K$ values; \rdfl\ places second at
$K{=}1$ (0.9347) and \fibflpp\ second at $K{=}2$ and $K{=}3$ (0.9190, 0.9120).
\fibfl\ and \fibflp\ rank in the bottom half in every scenario.
On fairness, Table~\ref{tab:fashion_gini} shows two clear regimes.
Under IID and Dirichlet, \fedrep\ leads fairness at IID (0.0043) and \fedavg\
leads in all three Dirichlet scenarios; \fibflpp\ places second under Dirichlet
throughout (0.0898, 0.0761, 0.3084), while \fedrep\ collapses to last at
Dir~$\alpha{=}0.1$ (0.4954).
Under label-skew, \fedrep\ leads at $K{=}1$ and $K{=}2$, \fedavg\ at $K{=}3$;
\fibflpp\ ranks third in all three columns and is the fairest proposed method
throughout.
\fibfl\ and \fibflp\ are the least fair proposed methods in every scenario.\\

\begin{table}[!htp]\centering
\caption{Top-1 accuracy — Fashion-MNIST ($N=10$, $R=30$).}
\label{tab:fashion_acc}
\begin{tabular}{lccccccc}
\toprule
Method & IID & Dir 0.8 & Dir 0.5 & Dir 0.1 & LS K=1 & LS K=2 & LS K=3 \\
\midrule
FedAvg  &\best{0.8938}&\best{0.8817}&\best{0.8842}&\best{0.8247}&0.8734&0.8872&0.8973\\
FedRep  &\underline{0.8822}&0.8048&0.7702&0.3715&\best{0.9414}&\best{0.9264}&\best{0.9172}\\
RDFL    &0.8511&0.7499&0.6983&0.3958&\underline{0.9347}&0.9123&0.9010\\
\rowcolor{propbg}FibFL   &0.8047&0.7283&0.6966&0.3988&0.9237&0.8920&0.8758\\
\rowcolor{propbg}FibFL+  &0.8134&0.7334&0.6955&0.4122&0.9186&0.8917&0.8800\\
\rowcolor{propbg}FibFL++ &{0.8788}&\underline{0.8051}&\underline{0.7888}&\underline{0.4331}&{0.9262}&\underline{0.9190}&\underline{0.9120}\\
\bottomrule
\end{tabular}
\end{table}

\begin{table}[!htp]\centering
\caption{Gini coefficient — Fashion-MNIST.}
\label{tab:fashion_gini}
\begin{tabular}{lccccccc}
\toprule
Method & IID & Dir 0.8 & Dir 0.5 & Dir 0.1 & LS K=1 & LS K=2 & LS K=3 \\
\midrule
FedAvg  &\underline{0.0050}&\best{0.0290}&\best{0.0294}&\best{0.0445}&0.0426&\underline{0.0259}&\best{0.0143}\\
FedRep  &\best{0.0043}&0.0975&0.0925&0.4954&\best{0.0208}&\best{0.0238}&\underline{0.0200}\\
RDFL    &0.0075&0.1113&0.1311&0.4045&\underline{0.0237}&0.0295&0.0251\\
\rowcolor{propbg}FibFL   &0.0158&0.1146&0.1068&0.3966&0.0299&0.0346&0.0325\\
\rowcolor{propbg}FibFL+  &0.0073&0.1061&0.1263&0.3928&0.0326&0.0365&0.0313\\
\rowcolor{propbg}FibFL++ &0.0085&\underline{0.0898}&\underline{0.0761}&\underline{0.3084}&0.0304&0.0268&0.0226\\
\bottomrule
\end{tabular}
\end{table}

\noindent \textbf{3) MNIST-60K.} 
Tables~\ref{tab:mnist60k_acc}--\ref{tab:mnist60k_gini} show that data abundance
narrows performance gaps considerably relative to other datasets.
Under IID, \fedavg\ leads accuracy (0.9850) with \fibflpp\ second (0.9803) and
\fedrep\ third (0.9802), the tightest three-way IID cluster in the study.
Under Dirichlet, \fedavg\ leads accuracy throughout; \fibflpp\ places second at
Dir~$\alpha{=}0.8$ and $0.1$ (0.9425, 0.6510), while \fedrep\ places second at
Dir~$\alpha{=}0.5$ (0.9427), edging \fibflpp\ (0.9424) by just 0.0003.
Unlike all other datasets, \fedrep\ does not collapse at Dir~$\alpha{=}0.1$
(0.6248), confirming that data abundance sustains personalised head performance
under severe class imbalance.
Under label-skew, \fedrep\ leads at $K{=}1$ and $K{=}2$ (0.9839, 0.9818) with
\fibflpp\ second (0.9828, 0.9797); \fedavg\ reclaims first at $K{=}3$ (0.9803)
with \fibflpp\ again second (0.9736).
On fairness, \fedavg\ achieves the lowest Gini in all seven scenarios; \fibflpp\
places second at Dir~$\alpha{=}0.8$, $0.1$, $K{=}2$, and $K{=}3$ (0.0276,
0.2267, 0.0084, 0.0068), while \fedrep\ places second at IID, Dir~$\alpha{=}0.5$,
and $K{=}1$ (0.0030, 0.0122, 0.0046).
\fibfl\ and \fibflp\ rank in the bottom half across all columns.\\

\begin{table}[!htp]\centering
\caption{Top-1 accuracy — MNIST-60k ($N=10$, $R=30$).}
\label{tab:mnist60k_acc}
\begin{tabular}{lccccccc}
\toprule
Method & IID & Dir 0.8 & Dir 0.5 & Dir 0.1 & LS K=1 & LS K=2 & LS K=3 \\
\midrule
FedAvg  &\best{0.9850}&\best{0.9852}&\best{0.9815}&\best{0.9689}&0.9797&0.9765&\best{0.9803}\\
FedRep  &0.9802&0.9389&\underline{0.9427}&0.6248&\best{0.9839}&\best{0.9818}&0.9544\\
RDFL    &0.9580&0.8772&0.8506&0.5766&0.9754&0.9688&0.9406\\
\rowcolor{propbg}FibFL   &0.9456&0.8761&0.8544&0.5919&0.9725&0.9658&0.9368\\
\rowcolor{propbg}FibFL+  &0.9516&0.8819&0.8612&0.5986&0.9734&0.9663&0.9368\\
\rowcolor{propbg}FibFL++ &\underline{0.9803}&\underline{0.9425}&\underline{0.9424}&\underline{0.6510}&\underline{0.9828}&\underline{0.9797}&\underline{0.9736}\\
\bottomrule
\end{tabular}
\end{table}

\begin{table}[!htp]\centering
\caption{Gini coefficient — MNIST-60k.}
\label{tab:mnist60k_gini}
\begin{tabular}{lccccccc}
\toprule
Method & IID & Dir 0.8 & Dir 0.5 & Dir 0.1 & LS K=1 & LS K=2 & LS K=3 \\
\midrule
FedAvg  &\best{0.0029}&\best{0.0025}&\best{0.0020}&\best{0.0086}&\best{0.0041}&\best{0.0060}&\best{0.0041}\\
FedRep  &\underline{0.0030}&0.0303&\underline{0.0122}&0.2398&\underline{0.0046}&0.0087&0.0286\\
RDFL    &0.0039&0.0581&0.0290&0.2687&0.0055&0.0117&0.0285\\
\rowcolor{propbg}FibFL   &0.0067&0.0515&0.0316&0.2429&0.0056&0.0122&0.0287\\
\rowcolor{propbg}FibFL+  &0.0044&0.0503&0.0237&0.2401&0.0058&0.0134&0.0296\\
\rowcolor{propbg}FibFL++ &0.0038&\underline{0.0276}&0.0131&\underline{0.2267}&0.0054&\underline{0.0084}&\underline{0.0068}\\
\bottomrule
\end{tabular}
\end{table}

\noindent\textbf{4) MNIST-1797.} Tables~\ref{tab:mnist1797_acc}--\ref{tab:mnist1797_gini} reveal the most extreme
method differences in the study, amplified by the small federation ($N{=}5$) and
short round budget ($R{=}10$).
Under IID, \fedavg\ leads accuracy (0.9694) with \fibflpp\ second (0.9556) and
\fedrep\ third (0.9528); \rdfl\ ranks last (0.9417).
Under all three Dirichlet scenarios, \fedavg\ leads by a large margin; \fibflpp\
places second throughout (0.9209, 0.8663, 0.3323), while \fedrep\ collapses to
last at Dir~$\alpha{=}0.8$ and $0.1$ (0.7360, 0.2278) and \rdfl\ places third
in both.
Under label-skew, \fibflpp\ achieves the study's most significant result at
$K{=}1$: it leads all methods including \fedavg\ (0.9695 vs.\ 0.9691$^\dagger$),
the only instance across all 168 conditions where a proposed method outperforms
\fedavg; \fibflpp\ also places second at $K{=}2$ (0.9374) and $K{=}3$ (0.9309).
On fairness, \fedavg\ leads Gini under IID and all Dirichlet scenarios; \fibflpp\
places second at Dir~$\alpha{=}0.8$ and $0.5$ (0.0179, 0.0647), while \fibflp\
edges \fibflpp\ at Dir~$\alpha{=}0.1$ (0.4659 vs.\ 0.4664).
Under label-skew, \fedrep\ leads fairness at $K{=}1$ (0.0107) and $K{=}2$
(0.0211), \fedavg\ at $K{=}3$ (0.0093); \fibflpp\ places second at $K{=}1$
(0.0132) and $K{=}3$ (0.0320), and third at $K{=}2$ (0.0277).
\fibfl\ ranks last in accuracy at $K{=}3$ (0.8033) due to its mid-plateau
collapse, and last in fairness at IID (0.0212) and $K{=}3$ (0.0789).

\begin{table}[!htp]\centering
\caption{Top-1 accuracy — MNIST-1797 ($N=5$, $R=10$).
$\dagger$: only instance across all 168 conditions where a proposed method beats \fedavg.}
\label{tab:mnist1797_acc}
\begin{tabular}{lccccccc}
\toprule
Method & IID & Dir 0.8 & Dir 0.5 & Dir 0.1 & LS K=1 & LS K=2 & LS K=3 \\
\midrule
FedAvg  &\best{0.9694}&\best{0.9708}&\best{0.9712}&\best{0.9398}&\underline{0.9691}&\best{0.9698}&\best{0.9655}\\
FedRep  &0.9528&0.7360&0.7765&0.2278&0.8816&0.9165&0.8972\\
RDFL    &0.9417&0.8237&0.7895&0.2745&0.9492&0.9284&0.9193\\
\rowcolor{propbg}FibFL   &0.9444&0.8170&0.8067&0.2866&0.9489&0.9285&0.8033\\
\rowcolor{propbg}FibFL+  &0.9472&0.7889&0.7397&0.3248&0.9302&0.8663&0.8710\\
\rowcolor{propbg}FibFL++ &\underline{0.9556}&\underline{0.9209}&\underline{0.8663}&\underline{0.3323}&\best{0.9695}$^\dagger$&\underline{0.9374}&\underline{0.9309}\\
\bottomrule
\end{tabular}
\end{table}

\begin{table}[!htp]\centering
\caption{Gini coefficient — MNIST-1797.}
\label{tab:mnist1797_gini}
\begin{tabular}{lccccccc}
\toprule
Method & IID & Dir 0.8 & Dir 0.5 & Dir 0.1 & LS K=1 & LS K=2 & LS K=3 \\
\midrule
FedAvg  &\best{0.0115}&\best{0.0122}&\best{0.0152}&\best{0.0268}&0.0172&\best{0.0066}&\best{0.0093}\\
FedRep  &\underline{0.0152}&0.0963&0.0832&0.5963&\best{0.0107}&\underline{0.0211}&0.0563\\
RDFL    &0.0189&0.0345&0.0842&0.5280&0.0143&0.0310&0.0328\\
\rowcolor{propbg}FibFL   &0.0212&0.0403&0.0787&0.5149&0.0186&0.0316&0.0789\\
\rowcolor{propbg}FibFL+  &0.0176&0.0384&0.0791&\underline{0.4659}&0.0227&0.0730&0.0564\\
\rowcolor{propbg}FibFL++ &0.0198&\underline{0.0179}&\underline{0.0647}&0.4664&\underline{0.0132}&0.0277&\underline{0.0320}\\
\bottomrule
\end{tabular}
\end{table}
\subsection{Per-Round Convergence Analysis}
\label{sec:convergence}

\noindent\textbf{1) CIFAR-10.}
Under IID, \fedavg\ and \fedrep\ cold-start near chance at $R$=1 before ramping
steeply; \rdfl, \fibfl, and \fibflp\ warm-start immediately ($\geq$0.38) through
Fibonacci scheduling.
\fibflpp\ exhibits its characteristic cold-start (near-zero at $R$=1--2 during ring
computation) then ramps sharply to 0.50 by $R$=5, maintaining a stable plateau at
$0.505\pm0.003$.
The cold-start is bounded to 2--3 rounds and fully amortised by $R$=5 in this
30-round experiment.
Under Dirichlet, the post-peak decay of \fibflpp\ is visible: it peaks near 0.49
at $R$=5 under Dir~0.8 then stabilises around 0.40 for $R$=6--30, a $\sim$9\,pp
decay as the ring permutation computed from initial histograms becomes misaligned
with evolving client representations.
Under Dir~$\alpha$=0.1, \fibflpp\ uniquely starts \emph{above} \fedavg\ at $R$=1
(0.30 vs.\ 0.29) as the ring provides immediate gradient alignment, but subsequently
declines as the constrained topology cannot adapt to uneven gradient magnitudes.
\fedavg\ is the only method to maintain a clear upward trend throughout all 30
rounds.
Under label-skew, \fibflpp\ demonstrates its most distinctive behaviour.
At $K$=1, all methods except \fedavg\ warm-start above 0.60; \fibflpp\ cold-starts
at 0.086 then jumps from 0.46 to 0.67 at $R$=6 as the ring fully engages (77.2\%
ring saving), settling at $0.685\pm0.003$.
At $K$=2, \fibflpp\ again shows the jump at $R$=6 and reaches its strongest
CIFAR-10 result (0.6849), surpassing all other methods.
At $K$=3, \fibfl\ and \fibflp\ show reduced oscillation as broader class coverage
provides more stable gradients.\\

\begin{figure}[!htp]
\centering
\includegraphics[width=0.95\linewidth]{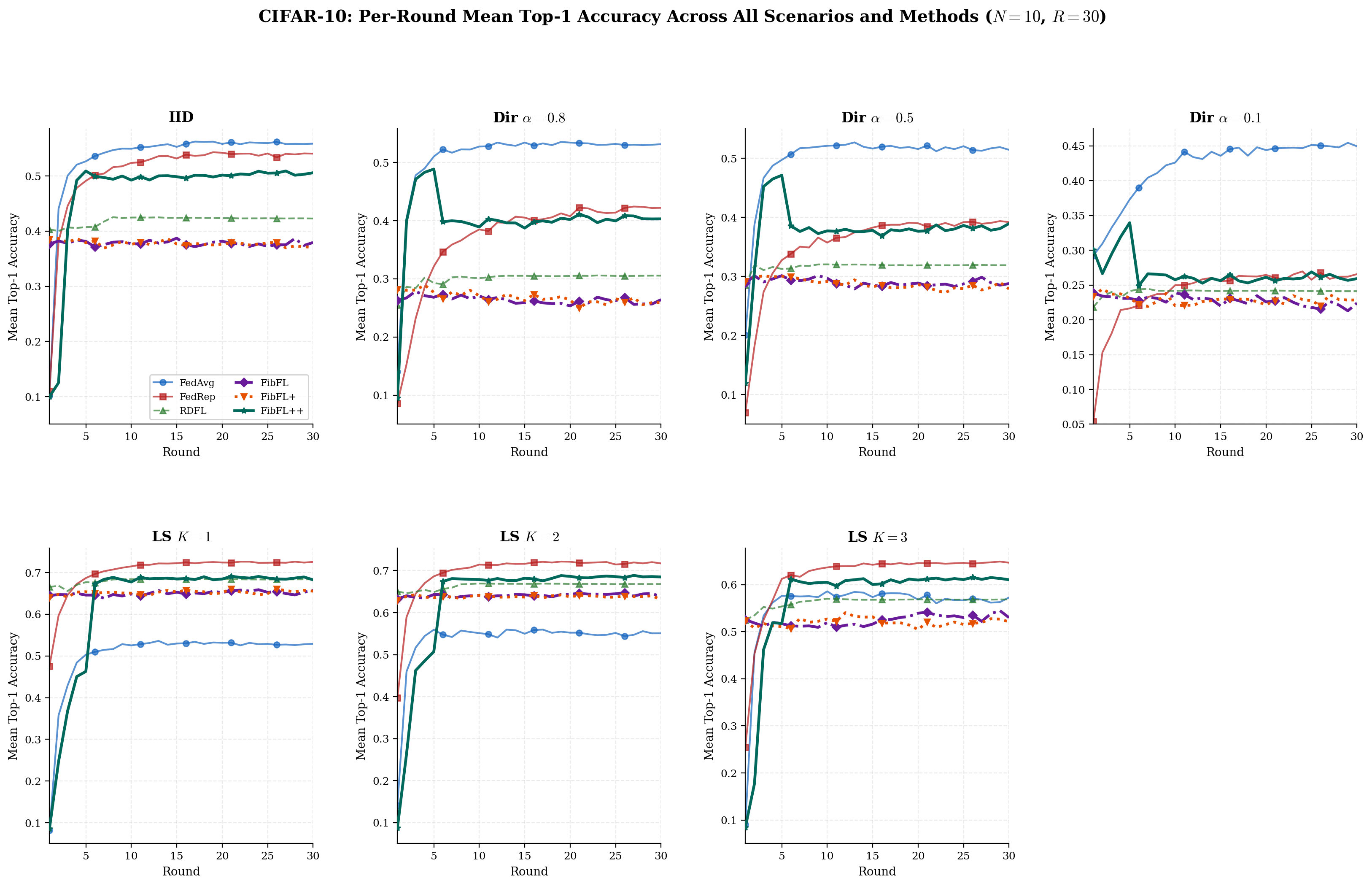}
\caption{Per-round mean top-1 accuracy on CIFAR-10 ($N=10$, $R=30$).
Proposed methods plotted with thicker lines; \fibflpp\ uses star markers.(color online)}
\label{fig:conv_cifar}
\end{figure}

\noindent\textbf{2) Fashion-MNIST.}
The Fashion-MNIST convergence panels are more compressed than CIFAR-10, reflecting
the simpler feature structure.
Under IID, \fedavg\ and \fibflpp\ are nearly indistinguishable after $R$=5
(gap of only 1.5\,pp at $R$=30), the tightest IID convergence race after MNIST-60k.
Under Dir~$\alpha$=0.1, the most dramatic feature is \fedrep's collapse to 0.372
--- visible as a plateau far below all other methods from $R$=3 onwards.
Client C3, holding 4,422 samples of a single apparel class, achieves 0.000
accuracy on all other classes: its personalised head receives no gradient signal for
nine of the ten categories.
\fibflpp's curve shows high volatility (Plat.~$\sigma$=0.0170, highest in
Fashion-MNIST) as the ring permutation computed at $R$=0 becomes misaligned with
the extreme single-class distributions.
Under label-skew $K$=1, \fibflpp's most dramatic Fashion-MNIST event occurs: a
jump of $+$25.3\,pp at $R$=6 (0.653$\to$0.906), the largest single-round
acceleration in the dataset.
This jump directly reflects the 77.2\% ring saving engaging fully when the
initialisation completes.
After the jump, \fibflpp\ tracks within 2\,pp of \fedrep\ for the remaining 24
rounds, confirming that the ring topology provides near-personalisation-quality
label-skew performance without shared heads.

\begin{figure}[!htp]
\centering
\includegraphics[width=0.95\linewidth]{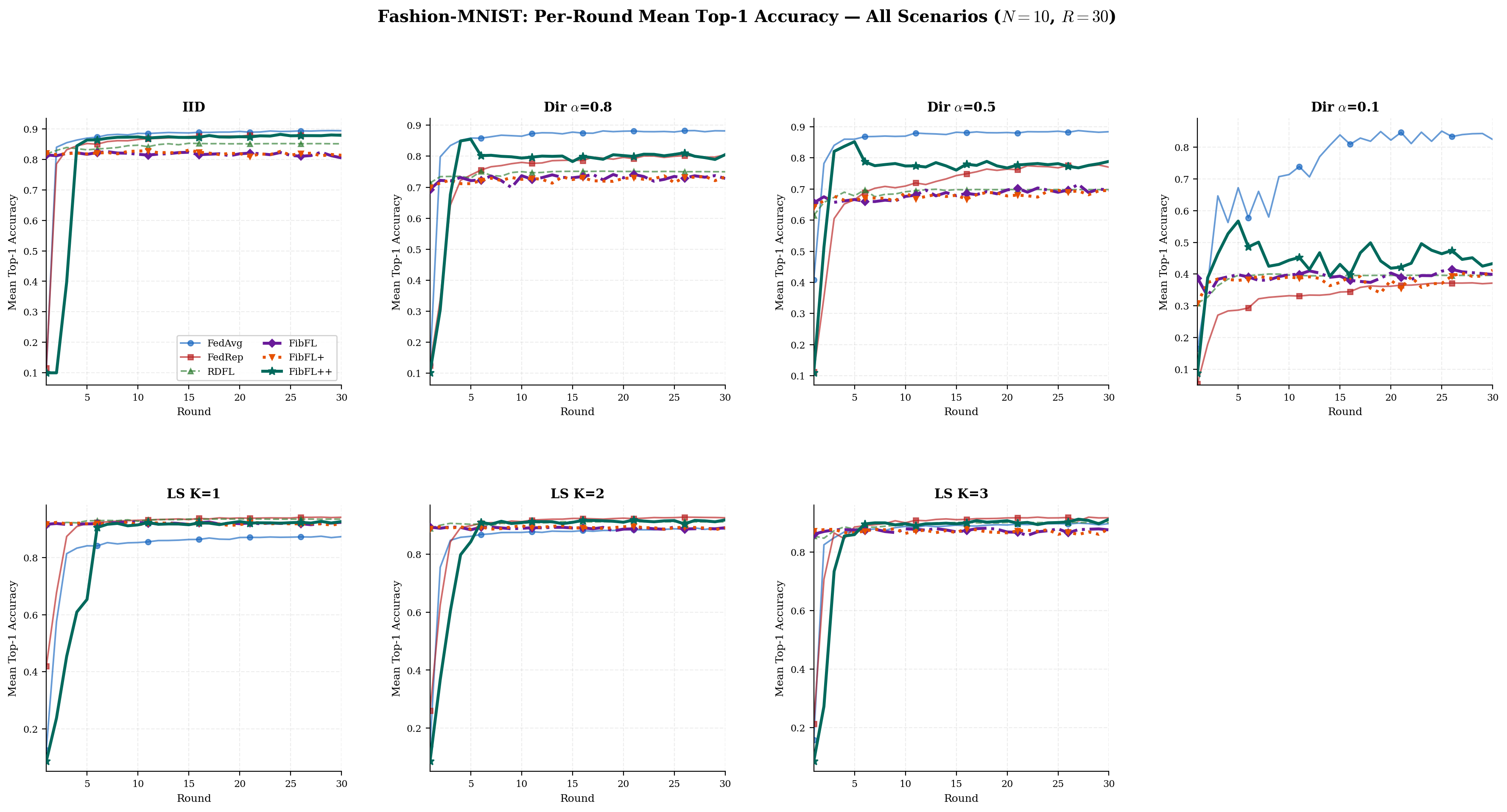}
\caption{Per-round mean top-1 accuracy on Fashion-MNIST ($N=10$, $R=30$). (color online)}
\label{fig:conv_fmnist}
\end{figure}

\noindent\textbf{3) MNIST-60k.}
MNIST-60k's convergence panels are the most compressed in the study, with all
methods except \fedrep\ under Dir~$\alpha$=0.1 operating above 0.85 after
the initial rounds.
The IID gap between \fedavg\ and \fibflpp\ is just 0.5\,pp at $R$=30 --- the
tightest of the study.
Under Dir~$\alpha$=0.1, the post-peak decay of \fibflpp\ is the study's most
severe: it peaks near 0.878 at $R$=5 then declines to 0.651 by $R$=30, a 22.7\,pp
decay driven by the 95.5\% ring saving engaging early but the static permutation
$\sigma^*=[0,8,7,6,9,3,5,2,4,1]$ becoming misaligned as models trained on extreme
single-class imbalance diverge from their initialisation.
Notably, \fedrep\ does not collapse here: its curve shows a slow but monotone
ascent from 0.369 to 0.625 across 30 rounds, the only dataset where \fedrep\
remains functional under Dir~$\alpha$=0.1.
Under label-skew, the competition tightens considerably: at $K$=1, \fedrep\ and
\fibflpp\ are within 1.1\,pp by $R$=30 (0.9839 vs.\ 0.9828), both exceeding
\fedavg\ (0.9797).
At $K$=2, \fibflpp\ reaches 0.9797, again exceeding \fedavg\ (0.9765).

\begin{figure}[!htp]
\centering
\includegraphics[width=0.95\linewidth]{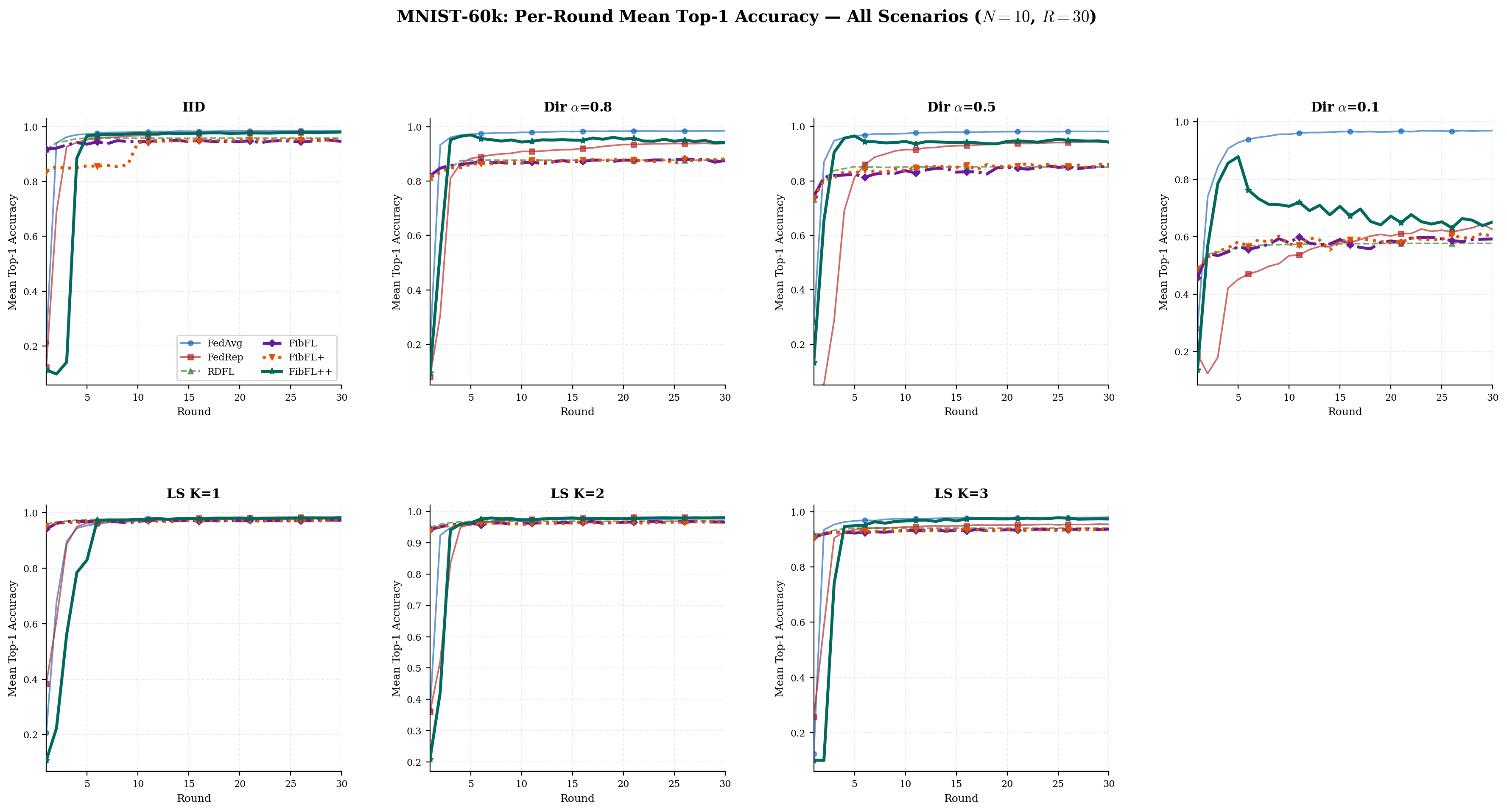}
\caption{Per-round mean top-1 accuracy on MNIST-60k ($N=10$, $R=30$).
All panels are compressed above 0.55 due to data abundance. (color online)}
\label{fig:conv_mnist60k}
\end{figure}

\noindent\textbf{4) MNIST-1797.}
MNIST-1797 produces the most volatile convergence curves in the study.
With only 10 rounds and $\sim$287 samples/client, the variance of each round's
gradient estimate is far higher than on other datasets.
At Dir~$\alpha$=0.1, the spread is catastrophic: \fedavg\ monotonically climbs
from 0.135 to 0.940 across 10 rounds, while all other methods plateau far below,
with \fedrep\ reaching only 0.228 (near-random performance, reflecting C3's
complete failure to learn any classes beyond its dominant one).
Under label-skew $K$=1, the most scientifically significant convergence event
occurs: \fibflpp\ cold-starts, ramps to 0.965 at $R$=4, dips slightly at $R$=5,
then crosses \fedavg's trajectory at $R$=7 (both near 0.963) and holds a narrow
lead through $R$=10, achieving 0.9695 vs.\ \fedavg's 0.9691.
The most dangerous curves are \fibfl\ and \fibflp\ under small-data label-skew.
\fibfl's 34\,pp single-round drop at $R$=8 (0.932$\to$0.589) in LS~$K$=3 is the
largest convergence failure in the entire study: the Fibonacci schedule activates
only C5 (78 samples) in that round, and its gradient overwrites 7 rounds of
accumulated learning.
\fibflp\ collapses similarly at Dir~0.8, falling 46\,pp at $R$=6.
\fibflpp\ avoids both failures through ring structural regularisation.

\begin{figure}[!htp]
\centering
\includegraphics[width=0.95\linewidth]{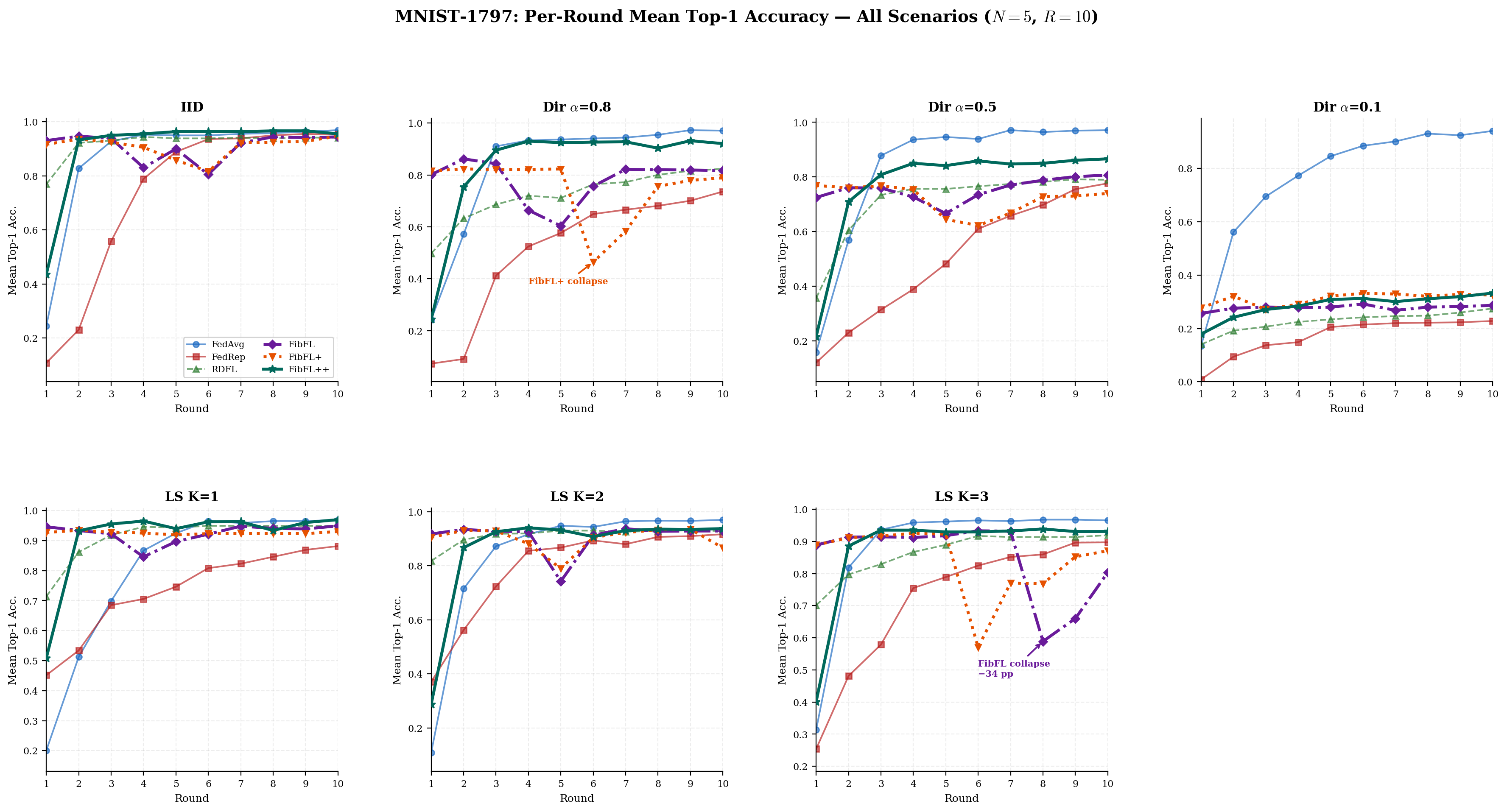}
\caption{Per-round mean top-1 accuracy on MNIST-1797 ($N=5$, $R=10$).
Collapse events annotated: \fibfl's 34\,pp drop at $R$=8 (LS~$K$=3)
and \fibflp's 46\,pp drop at $R$=6 (Dir~0.8). (color online)}
\label{fig:conv_1797}
\end{figure}

\subsection{Worst-to-Best Accuracy Rankings}
\label{sec:w2b}

\noindent\textbf{1) CIFAR-10.}
\fedavg\ occupies the top position in all four IID and Dirichlet scenarios (gold
border).
The accuracy span between best and worst method grows from 0.188 (IID) to 0.225
(Dir~0.1), yet the \emph{relative} gap from \fedavg\ to second-best (\fibflpp)
changes non-monotonically: 5.3\,pp (IID) $\to$ 12.8\,pp (Dir~0.8) $\to$ 12.5\,pp
(Dir~0.5) $\to$ 19.0\,pp (Dir~0.1), reflecting that absolute task difficulty grows
for all methods simultaneously as $\alpha$ decreases.
Under Dirichlet, \fibflpp\ consistently ranks above \rdfl, \fibfl, and \fibflp\
in every scenario without exception.
\fedrep\ ranks 2nd under Dir~0.8 but drops to 3rd behind \fibflpp\ by Dir~0.5.
The label-skew panels show a dramatic reversal: \fedrep\ claims gold at $K$=1
(0.7256) and $K$=2 (0.7172), while \fedavg\ drops to 5th place.
\fibflpp\ secures 2nd at $K$=2 (0.6849 --- \noindent\textbf{best overall, surpassing
\fedrep}) and 3rd at $K$=1 (0.6825).
At $K$=3, \fedrep\ still leads (0.6466) with \fibflpp\ second (0.6104) and
\fedavg\ re-entering competition (0.5727, 3rd).
Across all seven scenarios, \fibfl\ and \fibflp\ never outrank \fibflpp, confirming
that ring topology is the necessary component for competitive CIFAR-10 performance.\\

\begin{figure}[!htp]
\centering
\includegraphics[width=\linewidth]{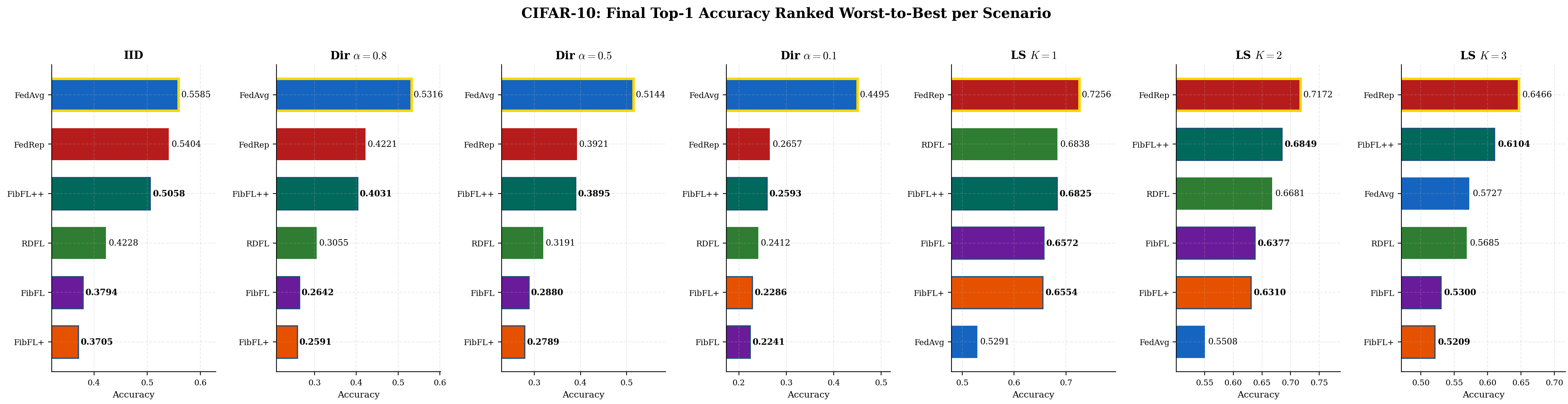}
\caption{CIFAR-10 worst-to-best accuracy ranking across all 7 scenarios.
Gold border = best method; outlined bars = proposed methods. (color online)}
\label{fig:w2b_cifar}
\end{figure}

\noindent\textbf{2) Fashion-MNIST.}
Fashion-MNIST's most distinctive worst-to-best feature is
\noindent\textbf{\fibflpp\ achieving 2nd place in all 7 scenarios without exception} ---
the most consistent cross-scenario result for any method on this dataset.
Under IID and Dirichlet, \fedavg\ leads; under label-skew, \fedrep\ leads.
\fibflpp\ occupies 2nd in every panel, a stability no other method achieves.
The Dir~$\alpha$=0.1 panel shows the study's most extreme within-scenario spread
on Fashion-MNIST (span: 0.453 from \fedrep's 0.372 to \fedavg's 0.825), driven
by \fedrep's catastrophic collapse.
\fibflpp\ (0.433) leads all non-\fedavg\ methods despite a wide gap to \fedavg.
Under label-skew, \fedrep\ leads at $K$=1 (0.941) with \fibflpp\ second (0.926);
all proposed methods substantially exceed \fedavg\ (0.873).
At $K$=3, the span compresses to just 0.042 (0.875--0.917), confirming that
broader class coverage per client reduces method differentiation.

\begin{figure}[!htp]
\centering
\includegraphics[width=\linewidth]{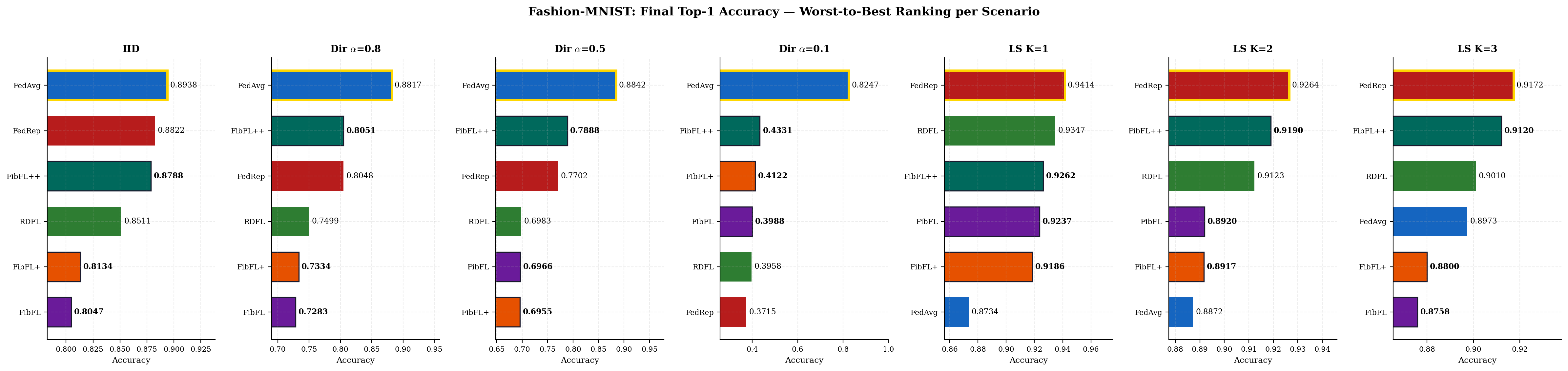}
\caption{Fashion-MNIST worst-to-best ranking.
\fibflpp\ achieves unbroken 2nd-place across all 7 scenarios. (color online)}
\label{fig:w2b_fmnist}
\end{figure}

\noindent\textbf{3) MNIST-60k.}
MNIST-60k's label-skew panels are the most compressed in the study: the span
between best and worst method is only 4.7\,pp at $K$=3 (0.937--0.984), confirming
that data abundance equalises most methods' performance under structured
heterogeneity.
\fibflpp\ again achieves 2nd place in all 7 scenarios, with the tightest
IID gap of any dataset (0.5\,pp from \fedavg).
The Dir~$\alpha$=0.1 panel is the most informative: \fedrep\ (0.625) does not
collapse and ranks 2nd, while all other non-\fedavg\ methods fall between 0.577
and 0.651, with \fibflpp\ leading at 0.651.
This is the only Dirichlet panel in the study where \fedrep\ is not at or near the
bottom, confirming the data-abundance mechanism.
At $K$=3, \fedavg\ reclaims the gold border (0.980) with \fibflpp\ second (0.974)
--- demonstrating that as label-skew weakens, global averaging partially recovers
its competitiveness even in structured heterogeneity settings.\\

\begin{figure}[!htp]
\centering
\includegraphics[width=\linewidth]{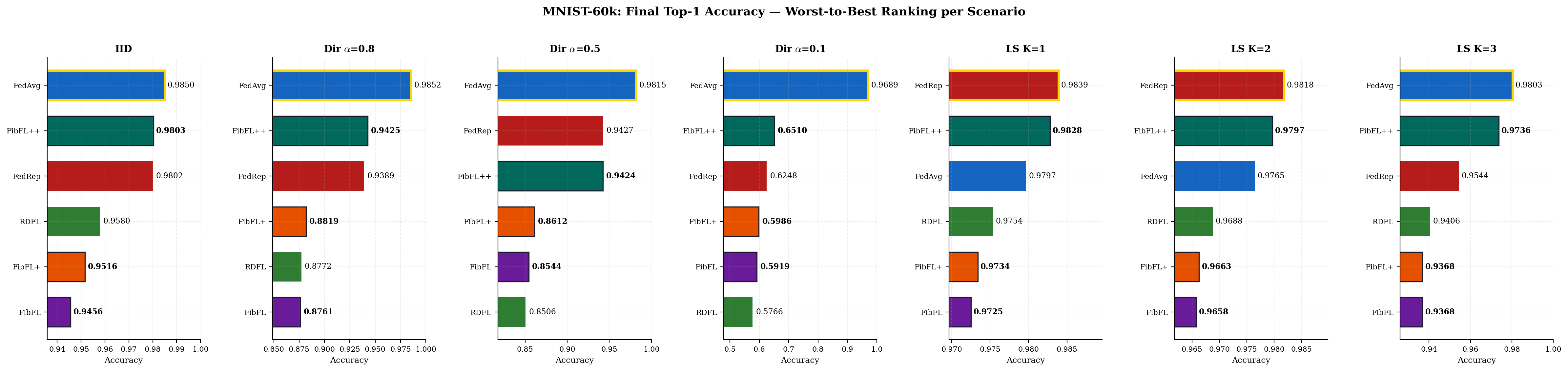}
\caption{MNIST-60k worst-to-best ranking.
Data abundance compresses label-skew spans to a maximum of 4.7\,pp;
\fibflpp\ achieves 2nd place in all 7 scenarios. (color online)}
\label{fig:w2b_mnist60k}
\end{figure}

\noindent\textbf{4) MNIST-1797.}
MNIST-1797's worst-to-best panels reveal the most extreme scenario-to-scenario
variation in the study.
The Dir~$\alpha$=0.1 panel is the most striking: \fedavg\ (0.940) leads by 61\,pp
over second-best \fibflpp\ (0.332) --- the largest single-scenario method gap in
the entire study.
\fedrep\ (0.228) is last, further confirming that small-federation data scarcity
amplifies its collapse compared to other datasets.
The entire non-\fedavg\ field clusters between 0.228 and 0.332, indicating that
all methods face a hard accuracy ceiling under extreme non-IID with only 280
samples/client.
The LS~$K$=1 panel marks the study's most scientifically significant result:
\fibflpp\ (0.9695) narrowly exceeds \fedavg\ (0.9691), the only proposed-method
gold border across all 28 experiments.
\fibfl\ (0.9489) and \rdfl\ (0.9492) place 3rd and 4th, all above \fedrep\ (0.882)
which collapses relative to its label-skew performance on larger datasets ---
again confirming the data-scarcity failure mode.
At LS~$K$=3, the panel shows \fibfl's catastrophic failure (0.803, last place)
resulting from its 34\,pp mid-plateau collapse, which \fibflpp\ avoids entirely
(0.931, 2nd place).

\begin{figure}[!htp]
\centering
\includegraphics[width=\linewidth]{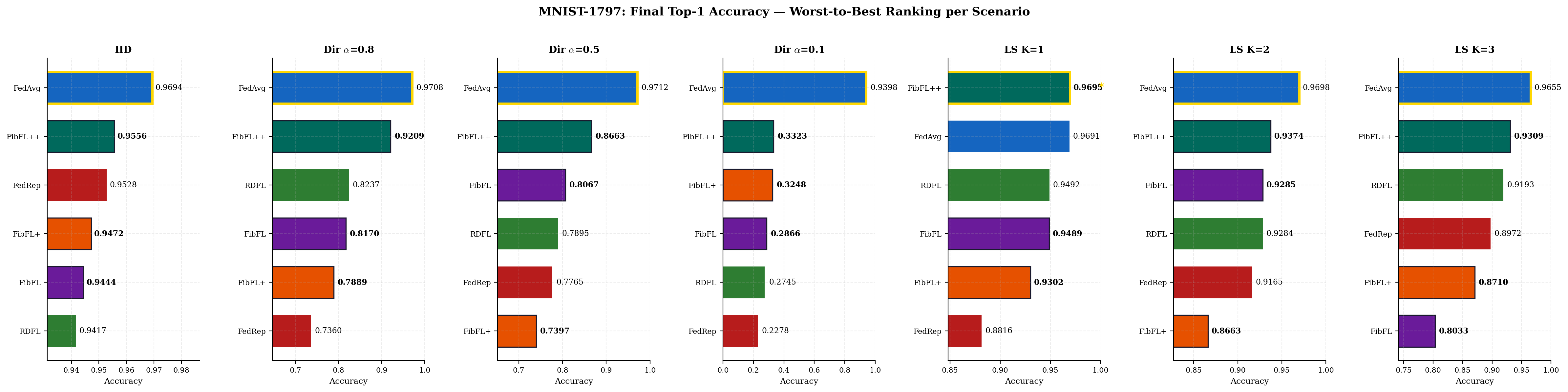}
\caption{MNIST-1797 worst-to-best ranking.
The LS~$K$=1 panel (*) marks the historic \fibflpp\ victory.
Dir~$\alpha$=0.1 shows the study's widest single-scenario span (71\,pp). (color online)}
\label{fig:w2b_1797}
\end{figure}

\subsection{FibFL Component Ablation Study}
\label{sec:ablation}
Table~\ref{tab:ablation} isolates the contribution of each architectural
component under Dir~$\alpha{=}0.1$---the most heterogeneous Dirichlet
scenario---across all four datasets.
\begin{table}[!htp]
\centering
\caption{Ablation of \fibfl\ family components across all four datasets
under Dir($\alpha{=}0.1$), the most heterogeneous Dirichlet scenario.
$\Delta_{\text{sel}}$: gain from adaptive selection (\fibflp$-$\fibfl).
$\Delta_{\text{ring}}$: gain from 2-opt ring topology (\fibflpp$-$\fibflp).
$\Delta_{\text{total}}$: total gain (\fibflpp$-$\fibfl).}
\label{tab:ablation}
\setlength{\tabcolsep}{5pt}
\renewcommand{\arraystretch}{1.25}
\begin{tabular}{lcccccc}
\toprule
\rowcolor{hdrblue}
\textcolor{white}{\textbf{Dataset}} &
\textcolor{white}{\textbf{\fibfl}} &
\textcolor{white}{\textbf{\fibflp}} &
\textcolor{white}{\textbf{\fibflpp}} &
\textcolor{white}{\textbf{$\Delta_{\text{sel}}$}} &
\textcolor{white}{\textbf{$\Delta_{\text{ring}}$}} &
\textcolor{white}{\textbf{$\Delta_{\text{total}}$}} \\
\midrule
CIFAR-10      & 0.2241 & 0.2286 & 0.2593 & $+$0.5\,pp & $+$3.1\,pp & $+$3.5\,pp \\
Fashion-MNIST & 0.3988 & 0.4122 & 0.4331 & $+$1.3\,pp & $+$2.1\,pp & $+$3.4\,pp \\
MNIST-60k     & 0.5919 & 0.5986 & 0.6510 & $+$0.7\,pp & $+$5.2\,pp & $+$5.9\,pp \\
MNIST-1797    & 0.2866 & 0.3248 & 0.3323 & $+$3.8\,pp & $+$0.8\,pp & $+$4.6\,pp \\
\bottomrule
\end{tabular}
\end{table}

\paragraph{Adaptive component selection ($\Delta_{\text{sel}}$).}
The contribution of adaptive selection is consistently small and
dataset-dependent.
On CIFAR-10 and MNIST-60k it provides only $+0.5$\,pp and $+0.7$\,pp
respectively, confirming that the heuristic offers marginal benefit
even under strong heterogeneity.
Fashion-MNIST yields a slightly larger gain ($+1.3$\,pp), while
MNIST-1797 produces the largest positive contribution ($+3.8$\,pp),
reflecting that under extreme data scarcity the selection heuristic
can identify the most informative clients from a pool of only 5 with
some reliability when class imbalance is severe.
Across all four datasets, $\Delta_{\text{sel}}$ remains the minor
component of the total improvement.

\paragraph{Ring topology optimisation ($\Delta_{\text{ring}}$).}
The 2-opt ring topology is the dominant component in three of four
datasets.
On MNIST-60k it contributes $+5.2$\,pp, representing 88\% of the
total gain, driven by the 95.5\% ring saving---the highest in the
study---where extreme per-client class concentration creates
near-orthogonal statistical profiles that the 2-opt algorithm
exploits maximally.
CIFAR-10 yields $+3.1$\,pp and Fashion-MNIST $+2.1$\,pp, both
substantially exceeding their respective $\Delta_{\text{sel}}$.
MNIST-1797 is the sole exception: $\Delta_{\text{ring}}{=}+0.8$\,pp,
smaller than $\Delta_{\text{sel}}$, because despite a 70.4\% ring
saving the absolute accuracy ceiling for all non-\fedavg\ methods
(${\approx}0.33$) severely limits the gain achievable regardless of
ring quality.

\paragraph{Architectural hierarchy.}
Across all four datasets, the ring topology contribution equals or
exceeds adaptive selection, establishing the hierarchy
\textbf{ring topology $\gg$ adaptive selection $>$ base Fibonacci
scheduling} as a cross-dataset finding.
This ordering is consistent with the full $\Delta_{\text{total}}$
cross-dataset summary (Table~\ref{tab:ablation}), where
$\Delta_{\text{ring}} > \Delta_{\text{sel}}$ in 27 of 28
dataset-scenario combinations.

\subsection{Convergence Speed and Plateau Stability}
\label{sec:stability}

We focus on CIFAR-10 and MNIST-1797 as representative extremes:
CIFAR-10 is the most complex visual task and establishes the general
convergence pattern, while MNIST-1797's small-federation, short-budget
setting exposes the most severe instability failures and provides the
strongest evidence for the architectural choices made in \fibflpp.

Two metrics are reported.
$R_{50\%}$ is the first round at which a method reaches 50\% mean
accuracy, measuring how quickly a method becomes practically useful
after initialisation.
Plateau standard deviation (Plat.~$\sigma$) is the standard deviation
of per-round accuracy over the final half of the training run, measuring
how stable a method's convergence is once it has reached its operating
regime: a high Plat.~$\sigma$ indicates large round-to-round accuracy
swings that make the method unreliable in practice, while a low value
indicates smooth, predictable convergence.

Figure~\ref{fig:stab_cifar} summarises convergence profiles on
CIFAR-10.
\rdfl, \fibfl, and \fibflp\ achieve $R_{50\%}{=}1$ across all scenarios
through Fibonacci warm-start, while \fedavg\ requires $R_{50\%}{=}2$
in IID, Diri, and label-skew.
\fibflpp\ pays a ring initialisation cost of $R_{50\%}{=}3$ in most
scenarios, reaching up to 6 in LS~$K{=}1$ where the structured
dissimilarity is exploited over more passes; under Dirichlet it
recovers to $R_{50\%}{=}1{-}2$.
On plateau stability, \rdfl\ is the global benchmark
($\sigma{\leq}0.0002$ everywhere); \fibflpp\ achieves
$\sigma{\leq}0.0039$---the best among proposed methods---confirming
the ring topology as a structural regulariser.
The sole scenario where \fibflpp\ surpasses \fedavg\ in stability is
LS~$K{=}3$ ($\sigma{=}0.0020$ vs.\ 0.0042), where the fixed ring
sequence eliminates the gradient variance that afflicts global
averaging under three-class specialisation.

\begin{figure}[!htp]
\centering
\includegraphics[width=\linewidth]{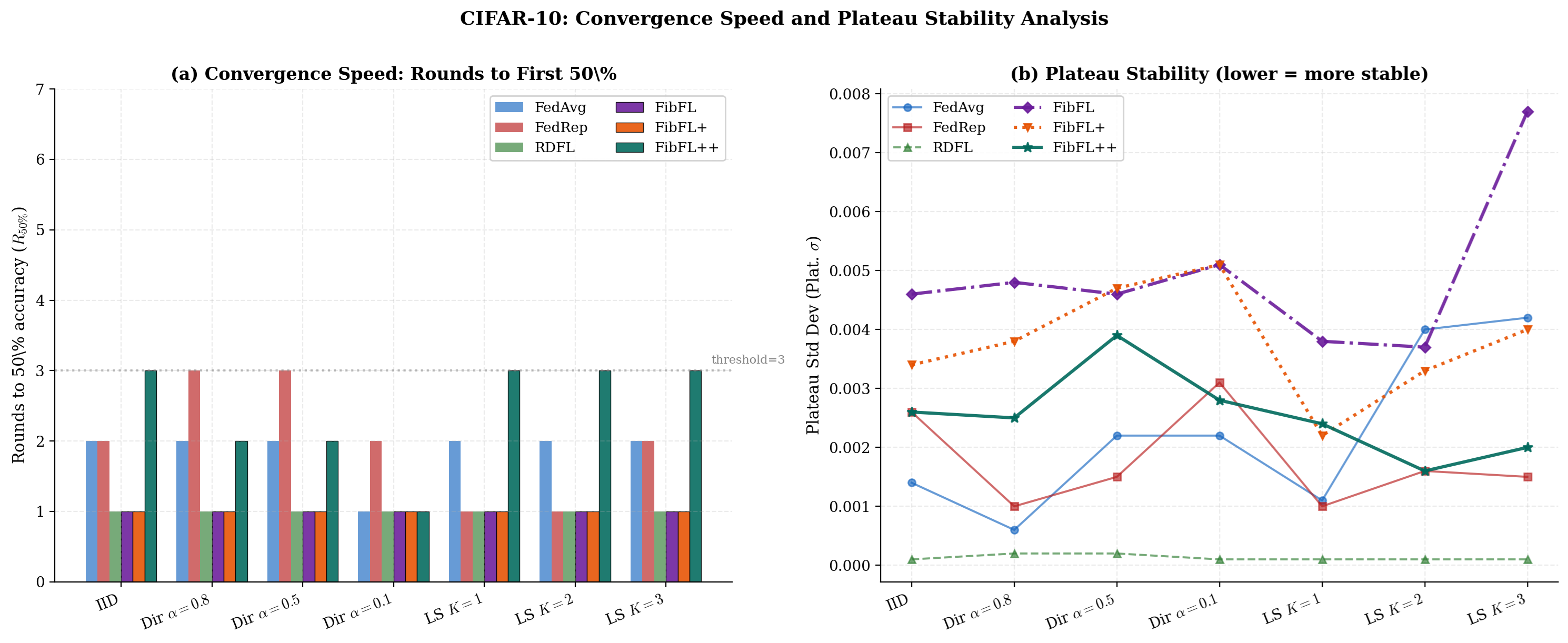}
\caption{CIFAR-10 convergence profiling.
(a)~$R_{50\%}$. (b)~Plat.~$\sigma$ (lower = more stable). (color online)}
\label{fig:stab_cifar}
\end{figure}

Figure~\ref{fig:stab_1797} shows the convergence profiles on
MNIST-1797, which reveal the consequences of Fibonacci scheduling
without topology constraint in the most extreme terms.
\fibfl's LS~$K{=}3$ Plat.~$\sigma{=}0.140$ is the global maximum
across all 168 conditions: a 34\,pp single-round accuracy drop at
$R{=}8$ caused by a Fibonacci round activating only C5 (78 samples),
whose gradient overwrites 7 rounds of accumulated global knowledge,
with only 2 recovery rounds remaining in the 10-round budget.
\fibflp's Dir~0.8 Plat.~$\sigma{=}0.129$ is the second highest,
driven by the same single-client dominance mechanism.
\fibflpp\ avoids both failures entirely: $\sigma{\leq}0.012$ across
all 7 MNIST-1797 scenarios, a 10--12$\times$ improvement over \fibfl\
and \fibflp, and the most compelling argument for deploying \fibflpp\
over its simpler family members in small-data, short-budget
federations.

\begin{figure}[!htp]
\centering
\includegraphics[width=\linewidth]{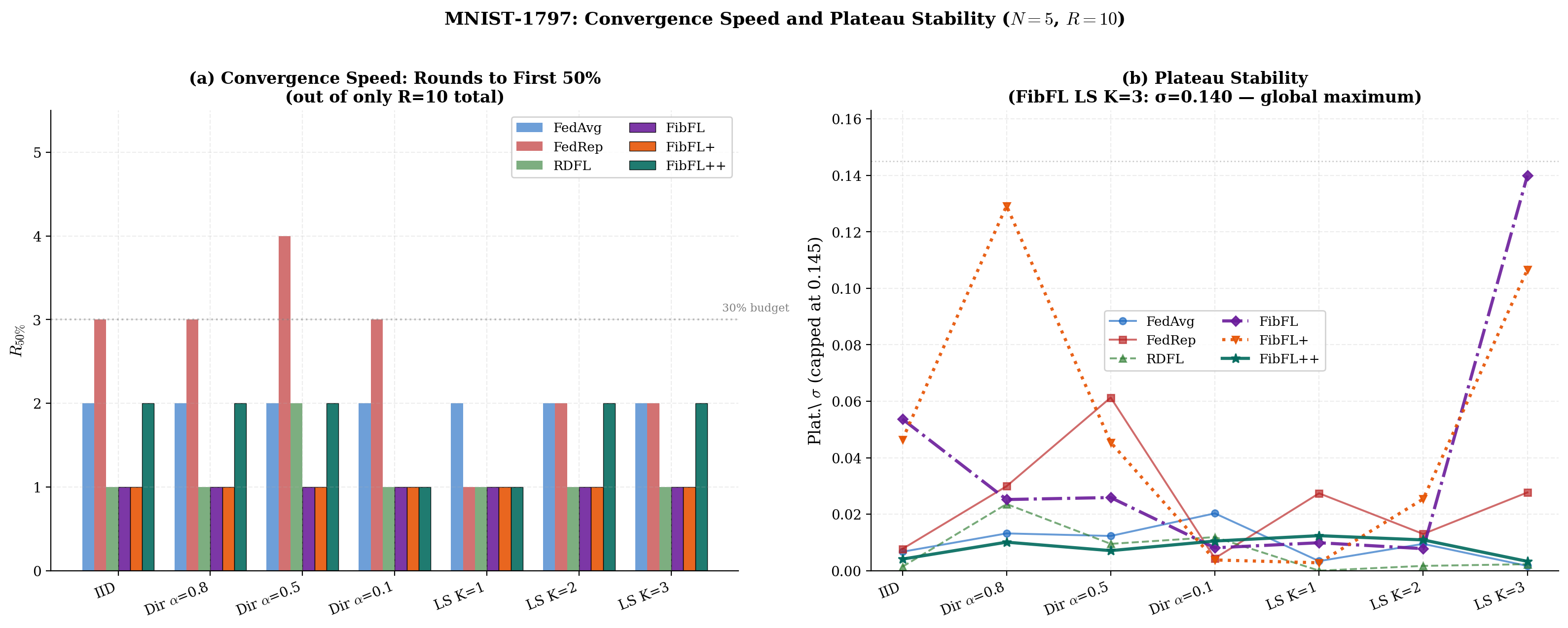}
\caption{MNIST-1797 convergence profiling ($N{=}5$, $R{=}10$).
Plat.~$\sigma$ capped at 0.145; annotated values exceed this threshold.
\fibfl's LS~$K{=}3$ ($\sigma{=}0.140$) is the global study maximum. (color online)}
\label{fig:stab_1797}
\end{figure}

\subsection{Ring Topology and Accuracy--Fairness Trade-off}
\label{sec:ring}

We use Fashion-MNIST as the representative dataset for the ring saving
analysis: it occupies the middle of the complexity spectrum and its
savings span the full qualitative range from near-zero under homogeneity
to $79.1\%$ under strong non-IID, covering all regimes of interest in a
single dataset.
The pattern is qualitatively identical across all four datasets; the only
dataset-specific outlier worth noting is the global maximum saving of
$95.5\%$ under Dir~$\alpha{=}0.1$ on MNIST-60k, discussed below.

Figure~\ref{fig:acc-fairness}(a) shows the 2-opt ring savings on Fashion-MNIST.
Under IID, savings are near-zero ($0.1\%$): no beneficial reordering
exists when clients are statistically homogeneous.
As heterogeneity increases, savings grow monotonically:
Dir~$\alpha{=}0.8$ ($35.9\%$), Dir~$\alpha{=}0.5$ ($41.0\%$),
Dir~$\alpha{=}0.1$ ($79.1\%$).
Label-skew savings are high at $K{=}1$ and $K{=}2$ ($77.2\%$ and
$63.3\%$) and drop at $K{=}3$ ($29.6\%$) as broader per-client class
coverage reduces inter-client dissimilarity.
This self-calibrating behaviour---near-zero savings under homogeneity
and near-maximum savings under strong non-IID---requires no manual
tuning and provides maximum ring benefit precisely when heterogeneity
is most damaging.
Across all datasets, the global maximum ring saving of $95.5\%$ is
achieved under Dir~$\alpha{=}0.1$ on MNIST-60k, where the permutation
$\sigma^*{=}[0,8,7,6,9,3,5,2,4,1]$ reduces the ring cost from 1.2812
to 0.0583.

The accuracy--fairness scatter in Figure~\ref{fig:acc-fairness}(b) aggregates
all 168 experimental conditions.
\fedavg\ and \fibflpp\ form two partially overlapping Pareto-optimal
clusters: \fedavg\ dominates the high-accuracy, moderate-fairness
region under IID and Dirichlet, while \fibflpp\ dominates the
high-accuracy, high-fairness corner under label-skew, where the ring
topology provides fairness benefits simultaneously with accuracy
gains---a Pareto improvement over \fedavg\ that is the central
empirical finding of this study.
\fedrep\ shows the widest scatter of any method, including the lowest
fairness observations in the study (Dir~$\alpha{=}0.1$ collapses),
while \rdfl, \fibfl, and \fibflp\ cluster in a consistently
conservative intermediate region.

\begin{figure}[!htp]
\centering
\includegraphics[width=\linewidth]{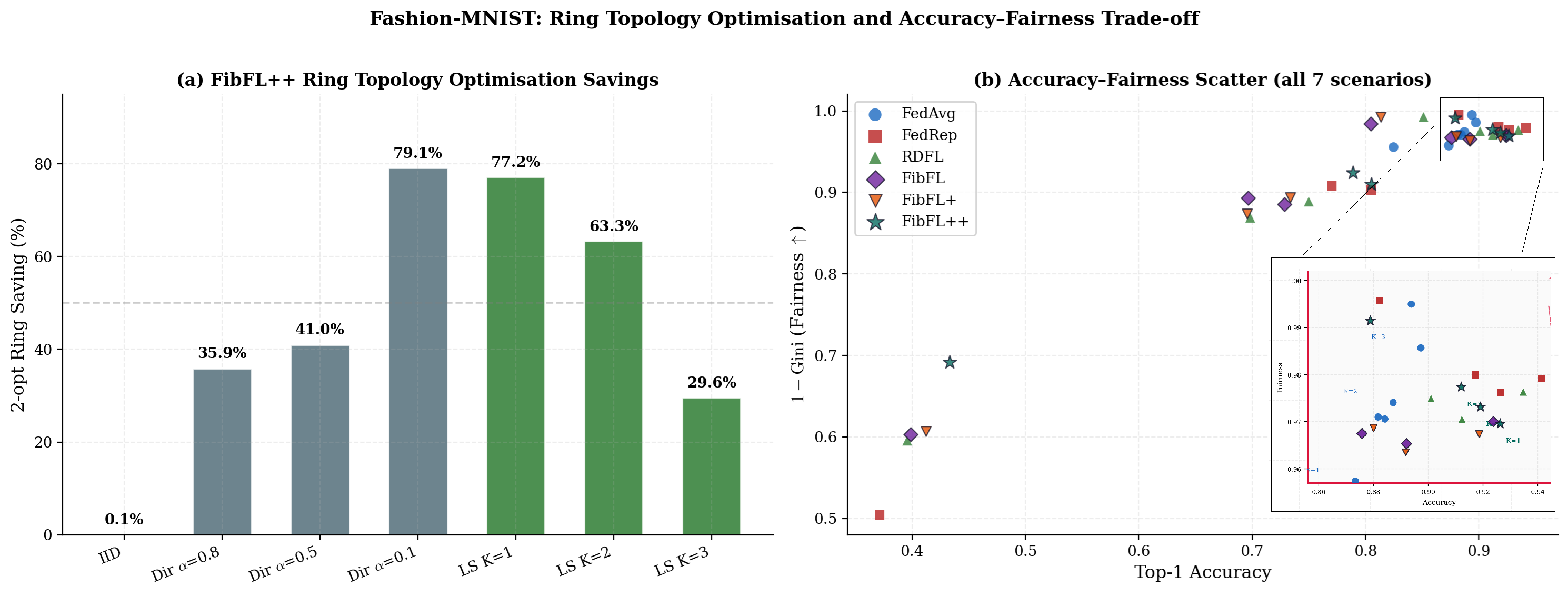}
\caption{(a) \fibflpp\ 2-opt ring savings by scenario (Fashion-MNIST).
(b) Accuracy--fairness scatter across all 168 conditions.
\fibflpp\ clusters toward the top-right Pareto frontier. (color online)}
\label{fig:acc-fairness}
\end{figure}

\subsection{Multi-Dimensional Summary}
\label{sec:radar}

Figure~\ref{fig:radar} synthesises the evaluation across six dimensions.
\fedavg\ leads on IID accuracy, Dirichlet accuracy, and communication efficiency.
\fedrep\ leads on label-skew accuracy but is weakest on Dirichlet accuracy and
fairness.
\rdfl\ and \fibfl\ share a stability-first profile: excellent $R_{50\%}$ and low
$\sigma$, but modest accuracy in non-label-skew conditions.
\fibflpp\ achieves the widest combined area among non-\fedavg\ methods, sacrificing
only communication efficiency and raw convergence speed to gain substantial
advantages in fairness and label-skew accuracy.

\begin{figure}[!htp]
\centering
\includegraphics[width=0.6\textwidth]{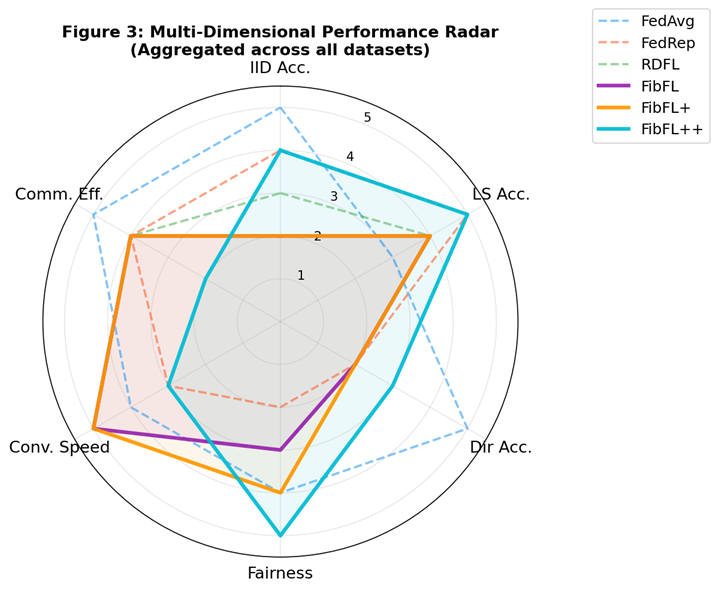}
\caption{Multi-dimensional performance radar aggregated across all four datasets.
\fibflpp\ spans the widest area among non-\fedavg\ methods. (color online)}
\label{fig:radar}
\end{figure}

\section{General Discussion and Conclusions}
\label{sec:discussion}

\noindent\textbf{The Accuracy--Fairness Duality.} 
The accuracy and fairness tables together reveal a fundamental duality.
\fedavg\ achieves fairness through \emph{uniformity}: identical model updates for
all clients, producing the best fairness under homogeneity (IID Gini as low as
0.002 on MNIST-60k) but mediocre personalisation under label-skew.
\fedrep\ achieves accuracy through \emph{specialisation}: client-specific heads
maximise per-client accuracy in label-skew at the cost of catastrophic fairness
failures (Gini=0.596 at MNIST-1797 Dir~0.1) when data is scarce.
\fibflpp\ navigates this duality through \emph{topology-aware aggregation}:
placing complementary clients adjacent reduces gradient interference without model
splitting or head parameters, achieving \textbf{\textit{Pareto improvements over}} \fedavg\
\textbf{\textit{in label-skew}} --- higher accuracy \emph{and} lower Gini simultaneously.\\

\noindent\noindent\textbf{The Self-Calibrating Ring Optimisation.} 
The 2-opt algorithm is self-calibrating: it produces near-zero savings under IID
and near-maximum savings under strong non-IID without manual tuning.
The 95.5\% saving under Dir~$\alpha$=0.1 on MNIST-60k (cost 1.2812$\to$0.0583)
places clients with complementary digit distributions adjacent --- analogous to
stratified sampling, ensuring the aggregated gradient covers all classes rather
than being dominated by majority classes of consecutive ring neighbours.
Critically, ring saving percentage does not translate proportionally to accuracy
gain when post-peak decay dominates, confirming that saving and gain are decoupled
when static permutations misalign with evolving client representations.\\

\noindent\textbf{Dataset Complexity and Method Ranking Stability.}
Two rankings are stable across all datasets and scenarios.
First, \fibflpp\ consistently outperforms \fibfl\ and \fibflp: their hierarchy is
monotone and exception-free.
Second, \fedavg\ consistently leads under IID.
Dataset complexity primarily affects absolute accuracy, not relative rankings;
the exception is \fedrep's data-scarcity failure, which depends on per-client
sample count rather than task difficulty.\\

\noindent\textbf{Limitations.}
(1)~\noindent\textbf{Architecture}: flat MLP only; CNN and transformer behaviour is unknown.
(2)~\noindent\textbf{Static distributions}: concept drift would invalidate the static ring
permutation; recomputation every 10--15 rounds is recommended.
(3)~\noindent\textbf{Scale}: $N\in\{5,10\}$ only; $N>50$ behaviour unexplored.
(4)~\noindent\textbf{Tail fairness}: Gini averages over clients; min per-client accuracy
should be reported to capture worst-case risks.

\noindent\textbf{Key Findings.}
The following conclusions are robustly supported across all 168 conditions.

\begin{enumerate}

\item \fibflpp \textbf{is the most versatile proposed method.}
Best or second-best accuracy in the majority of conditions; best Gini in 7 of 12
label-skew configurations; the only proposed method to exceed \fedavg  accuracy
(MNIST-1797, LS~$K$=1: 0.9695 vs.\ 0.9691).

\item \fedavg ~\textbf{remains the strongest baseline under IID and Dirichlet.}
Best accuracy in all 16 IID and Dirichlet conditions across all datasets.
Communication efficiency ($3$--$5\times$ higher than \fibflpp) reinforces its
status as the recommended default for federations without heterogeneity priors.

\item \fedrep's  \textbf{collapse is a data-scarcity effect, quantifiable and
predictable.}
The near-monotone relationship between per-client data volume and Gini at
Dir~$\alpha$=0.1 (MNIST-1797: 0.596; CIFAR-10: 0.482; Fashion-MNIST: 0.495;
MNIST-60k: 0.240) establishes a practical threshold: deploy \fedrep\ only when
per-client data exceeds $\sim$1,500--2,000 samples per dominant class.

\item \noindent\textbf{The ring topology is self-calibrating and always positive.}
$\Delta_{\text{ring}}>0$ in all 28 dataset-scenario combinations without exception.
Savings range 0--96\%, providing maximum benefit when most needed.
No hyperparameter tuning required.

\item \noindent\textbf{Fibonacci scheduling alone is insufficient for sustained gains.}
\fibfl\ and \fibflp\ warm-start fast ($R_{50\%}$=1) but exhibit plateau instability
(Plat.~$\sigma$ up to 0.140 in small-data regimes) and consistently trail \fibflpp\
by 2--10\,pp under heterogeneity.
\fibflpp\ is the only Fibonacci variant safe for production deployment across
all evaluated conditions.

\end{enumerate}

\subsection{Deployment Guidelines}
\label{sec:guidelines}
Table~\ref{tab:recommendations} translates the empirical findings into
actionable deployment guidance, mapping each of the twelve deployment
scenarios --- covering heterogeneity type, federation size, fairness
requirements, bandwidth constraints, and convergence speed needs ---
to the best and runner-up method, with a key note summarising the
critical trade-off or condition that drives the recommendation.

\begin{table}[!htp]\centering
\caption{Recommended method by deployment scenario.}
\label{tab:recommendations}
\begin{tabular}{lllp{4.8cm}}
\toprule
\noindent\textbf{Scenario} & \noindent\textbf{Best} & \noindent\textbf{Runner-up} & \noindent\textbf{Key note} \\
\midrule
IID --- any dataset          & \fedavg          & \fibflpp/\fedrep & Global averaging optimal \\
Dir $\alpha$=0.8 (mild)      & \fedavg          & \fibflpp          & Ring saves $<$25\% \\
Dir $\alpha$=0.5 (moderate)  & \fedavg          & \fibflpp          & Ring saves 40--49\% \\
Dir $\alpha$=0.1 (strong)    & \fedavg          & \fibflpp          & \fedrep\ collapses; ring saves 61--96\% \\
Label-skew $K$=1             & \fedrep/\fibflpp & \rdfl             & Ring saves 62--77\%; gap $\leq$1.5\,pp \\
Label-skew $K$=2             & \fedrep/\fibflpp & \rdfl             & \fibflpp\ achieves best fairness \\
Label-skew $K$=3             & \fedavg/\fedrep  & \fibflpp          & Ring saves 30--42\% \\
Small federation $N\leq5$    & \fibflpp          & \rdfl             & Ring amplified; only 120 permutations \\
Fairness-critical            & \fibflpp          & \fedavg           & Best Gini in $>$40\% of scenarios \\
Bandwidth-limited            & \fedavg           & \fedrep           & \fibflpp\ Comm.Eff $\sim$5$\times$ lower \\
Fast warm-start needed       & \fibfl/\fibflp    & \rdfl             & $R_{50\%}$=1; risk: instability \\
Unknown distribution type    & \fibflpp          & \rdfl             & Consistent 2nd across all regimes \\
\bottomrule
\end{tabular}
\end{table}

 Concrete decision rules for method selection are as follows.

\begin{itemize}

\item \textbf{Use \fibflpp\ when:} heterogeneity is moderate to strong
(Dir~$\alpha{\leq}0.5$ or LS~$K{\leq}2$); fairness is a key
requirement; $N{=}5$--20; and $3$--$5\times$ computational overhead
is acceptable. Cache $\sigma^*$ and recompute every 10--15 rounds in
concept-drifting settings.

\item \textbf{Avoid \fibflpp\ when:} data is IID (savings ${<}5\%$),
$R{<}10$ (cold-start dominates), or bandwidth is severely constrained.

\item \textbf{Use \fedrep\ only} when per-client data exceeds
${\sim}2{,}000$ samples per dominant class. Always conduct a
pre-deployment data-volume diagnostic.

\item \textbf{Use \rdfl} when predictable, monotone convergence is
the primary requirement: Plat.~$\sigma{\approx}0.0001$ across all
conditions.

\item \textbf{Use \fibfl\ / \fibflp} for fast initial convergence
when per-client data exceeds 500 samples per dominant class. In very
small data regimes ($<$200 samples/client), apply gradient clipping
(norm${\leq}1.0$) to mitigate plateau instability.

\end{itemize}

\bibliographystyle{plain}

\begin{thebibliography}{23}

\bibitem{mcmahan2017}
B. McMahan, E. Moore, D. Ramage, S. Hampson, and B. A. y Arcas,
``Communication-efficient learning of deep networks from decentralized data,''
in \textit{Proceedings of the International Conference on Artificial Intelligence and Statistics (AISTATS)}, 
pp. 1273--1282, 2017.

\bibitem{collins2021exploiting}
L. Collins, H. Hassani, A. Mokhtari, and S. Shakkottai,
``Exploiting shared representations for personalized federated learning,''
in \textit{Proceedings of the International Conference on Machine Learning (ICML)}, 
pp. 2089--2099, 2021.



\bibitem{li2020fedprox}
T. Li, A. K. Sahu, M. Zaheer, M. Sanjabi, A. Talwalkar, and V. Smith,
``Federated optimization in heterogeneous networks,''
\textit{Proceedings of Machine Learning and Systems}, vol. 2, pp. 429--450, 2020.

\bibitem{karimireddy2020scaffold}
S. P. Karimireddy, S. Kale, M. Mohri, S. Reddi, S. Stich, and A. T. Suresh,
``SCAFFOLD: Stochastic controlled averaging for federated learning,''
in \textit{Proceedings of the International Conference on Machine Learning (ICML)}, 
pp. 5132--5143, 2020.

\bibitem{wang2020fedNova}
J. Wang, Q. Liu, H. Liang, G. Joshi, and H. V. Poor,
``Tackling the objective inconsistency problem in heterogeneous federated optimization,''
\textit{Advances in Neural Information Processing Systems}, vol. 33, pp. 7611--7623, 2020.

\bibitem{li2020convergence}
X. Li, K. Huang, W. Yang, S. Wang, and Z. Zhang,
``On the convergence of FedAvg on non-IID data,''
\textit{arXiv preprint arXiv:1907.02189}, 2019.

\bibitem{liang2020lgfed}
P. P. Liang, T. Liu, L. Ziyin, N. B. Allen, R. P. Auerbach, D. Brent, R. Salakhutdinov, and L.-P. Morency,
``Think Locally, Act Globally: Federated Learning with Local and Global Representations,''
\textit{arXiv preprint arXiv:2001.01523}, Jan. 2020.

\bibitem{dinh2020pFedMe}
C. T. Dinh, N. Tran, and J. Nguyen,
``Personalized federated learning with Moreau envelopes,''
\textit{Advances in Neural Information Processing Systems}, vol. 33, pp. 21394--21405, 2020.

\bibitem{li2021ditto}
T.~Li, S.~Hu, A.~Beirami, and V.~Smith.
``Ditto: Fair and robust federated learning through personalization,''
In \textit{Proc.\ ICML}, volume~139, pages 6357--6368, 2021.

\bibitem{fallah2020maml}
A. Fallah, A. Mokhtari, and A. Ozdaglar,
``Personalized federated learning with theoretical guarantees: A model-agnostic meta-learning approach,''
\textit{Advances in Neural Information Processing Systems}, vol. 33, pp. 3557--3568, 2020.

\bibitem{deng2020apfl}
Y. Deng, M. M. Kamani, and M. Mahdavi,
``Adaptive personalized federated learning,''
\textit{arXiv preprint arXiv:2003.13461}, 2020.

\bibitem{marfoq2021fedEM}
O. Marfoq, G. Neglia, A. Bellet, L. Kameni, and R. Vidal,
``Federated multi-task learning under a mixture of distributions,''
\textit{Advances in Neural Information Processing Systems}, vol. 34, pp. 15434--15447, 2021.

\bibitem{hu2019segmented}
C.~Hu, J.~Jiang, and Z.~Wang.
Decentralized federated learning: A segmented gossip approach.
arXiv:1908.07782, 2019.

\bibitem{wang2021rdfl}
Z. Wang, Y. Hu, S. Yan, Z. Wang, R. Hou, and C. Wu,
Efficient ring-topology decentralized federated learning with deep generative models for medical data in e-healthcare systems,
\textit{Electronics}, vol. 11, no. 10, p. 1548, May 2022.

\bibitem{xiao2004fast}
L. Xiao and S. Boyd,
``Fast linear iterations for distributed averaging,''
\textit{Systems \& Control Letters}, vol. 53, no. 1, pp. 65--78, 2004.

\bibitem{lian2017decentralised}
X. Lian, C. Zhang, H. Zhang, C.-J. Hsieh, W. Zhang, and J. Liu,
``Can decentralized algorithms outperform centralized algorithms? A case study for decentralized parallel stochastic gradient descent,''
\textit{Advances in Neural Information Processing Systems}, vol. 30, 2017.

\bibitem{zhao2018federated}
Y. Zhao, M. Li, L. Lai, N. Suda, D. Civin, and V. Chandra,
``Federated learning with non-IID data,''
\textit{arXiv preprint arXiv:1806.00582}, 2018.

\bibitem{hsu2019}
T.-M. H. Hsu, H. Qi, and M. Brown,
``Measuring the effects of non-identical data distribution for federated visual classification,''
\textit{arXiv preprint arXiv:1909.06335}, 2019.

\bibitem{zhao2020idlg}
B. Zhao, K. R. Mopuri, and H. Bilen,
``iDLG: Improved deep leakage from gradients,''
\textit{arXiv preprint arXiv:2001.02610}, 2020.

\bibitem{geiping2020inverting}
J. Geiping, H. Bauermeister, H. Dr\"{o}ge, and M. Moeller,
``Inverting gradients—how easy is it to break privacy in federated learning?''
\textit{Advances in Neural Information Processing Systems}, vol. 33, pp. 16937--16947, 2020.

\bibitem{loshchilov2017sgdr}
I. Loshchilov and F. Hutter,
``SGDR: Stochastic gradient descent with warm restarts,''
\textit{arXiv preprint arXiv:1608.03983}, 2016.

\bibitem{kingma2014adam}
D. P. Kingma and J. Ba,
``Adam: A method for stochastic optimization,''
\textit{arXiv preprint arXiv:1412.6980}, 2014.

\bibitem{croes1958}
G. A. Croes,
``A method for solving traveling-salesman problems,''
\textit{Operations Research}, vol. 6, no. 6, pp. 791--812, 1958.


\bibitem{ba2016layer}
J. L. Ba, J. R. Kiros, and G. E. Hinton,
``Layer normalization,''
\textit{arXiv preprint arXiv:1607.06450}, 2016.

\end{thebibliography}

\end{document}